\documentclass[journal]{IEEEtran}

\usepackage{subfigure}

\usepackage{amsmath,amssymb,amsfonts}
\usepackage{mathrsfs}
\usepackage{multicol,lipsum}
\usepackage{mathtools, cuted,color}

\usepackage{amsthm}
\usepackage{enumerate}  
\theoremstyle{plain}
\newtheorem{theorem}{\textbf{Theorem}}[section]
\theoremstyle{plain}
\newtheorem{definition}{\textbf{Definition}}[section]
\theoremstyle{plain}
\newtheorem{lemma}{\textbf{Lemma}}[section]

\usepackage{cite}

\usepackage{algorithmic}
\usepackage[ruled,vlined]{algorithm2e}

\usepackage{graphicx}

\usepackage{balance}
\usepackage{soul}
\soulregister\cite7
\soulregister\ref7

\usepackage[colorinlistoftodos,prependcaption,textwidth=0.8cm,textsize=small]{todonotes}   
\newcommand{\respone}[1]{\todo[linecolor=orange,backgroundcolor=orange!25,bordercolor=orange]{#1}}
\newcommand{\resptwo}[1]{\todo[linecolor=red,backgroundcolor=red!25,bordercolor=red]{#1}} 
\newcommand{\respthree}[1]{\todo[linecolor=green,backgroundcolor=green!25,bordercolor=green]{#1}}
\newcommand{\resped}[1]{\todo[linecolor=blue,backgroundcolor=blue!25,bordercolor=blue]{#1}}

\usepackage{ifthen}
\newboolean{showmodification}
\setboolean{showmodification}{false}

\makeatletter
\ifthenelse{\boolean{showmodification}}{
}
{
	\AtBeginDocument{\let\hl\@firstofone}
	\AtBeginDocument{\let\hlmath\@firstofone}	
	\renewcommand\respone[1]{}
	\renewcommand\resptwo[1]{}
	\renewcommand\respthree[1]{}
	\renewcommand\resped[1]{}
}

\usepackage{soul}
\renewcommand{\st}[1]{}

\makeatother

\ifodd 0

\newcommand{\rev}[1]{{\color{blue}#1}} 
\newcommand{\com}[1]{\textbf{\color{red}(COMMENT: #1)}} 
\newcommand{\mcom}[1]{\textbf{\color{purple}(Response: #1)}} 
\newcommand{\edt}[1]{\textbf{\color{magenta}#1}} 
\newcommand{\clar}[1]{\textbf{\color{green}(NEED CLARIFICATION: #1)}}

\else
\newcommand{\rev}[1]{#1}

\newcommand{\com}[1]{}
\newcommand{\mcom}[1]{}
\newcommand{\edt}[1]{}
\newcommand{\clar}[1]{}
\fi

\begin{document}

\title{\LARGE \bf
	Recycled ADMM: Improving the Privacy and Accuracy of Distributed Algorithms}

\author{Xueru Zhang, Mohammad Mahdi Khalili, Mingyan Liu
	\thanks{{This work is supported by the NSF under grants CNS-1422211, CNS-1646019, and CNS-1739517. An earlier version of this paper appeared in the 2018 Allerton Conference on Communication, Control and Computing\cite{zhang2018recycled}.}}
	\thanks{X. Zhang, M. Khalili and M. Liu are with the Dept. of Electrical Engineering and Computer Science, University of Michigan, Ann Arbor, MI 48105, \{xueru, khalili, mingyan\}@umich.edu}}

\maketitle
\begin{abstract}
Alternating direction method of multiplier (ADMM) is a powerful method to solve decentralized convex optimization problems. In distributed settings, each node performs computation with its local data and the local results are exchanged among neighboring nodes in an iterative fashion. During this iterative process the leakage of data privacy arises and can accumulate significantly over many iterations, making it difficult to balance the privacy-accuracy tradeoff. We propose Recycled ADMM (R-ADMM), where a linear approximation is applied to every even iteration, its solution directly calculated using only results from the previous, odd iteration. It turns out that under such a scheme, half of the updates incur no privacy loss and require much less computation compared to the conventional ADMM. Moreover, R-ADMM can be further modified (MR-ADMM) such that each node independently determines its own penalty parameter over iterations. We obtain a sufficient condition for the convergence of both algorithms and provide the privacy analysis based on objective perturbation. It can be shown that the privacy-accuracy tradeoff can be improved significantly compared with conventional ADMM.

\end{abstract}


%
\IEEEpeerreviewmaketitle

\section{Introduction}\label{sec:intro}
\IEEEPARstart{D}{istributed} optimization and learning are crucial for many settings where the data is possessed by multiple parties or when the quantity of data prohibits processing at a central location. 
Many problems can be formulated as a convex optimization of the following form: $\min_\textbf{x}\sum_{i=1}^{N}f_i(\textbf{x})$.
In a distributed setting, each entity/node $i$ has its own local objective $f_i$, $N$ entities/nodes collaboratively work to solve this objective through an interactive process of local computation and message passing. At the end all local results should ideally converge to the global optimum.

The information exchanged over the iterative process gives rise to privacy concerns if the local training data contains sensitive information such as medical or financial records, web search history, and so on \cite{extra2,Khalili,extra1,magazine}. It is therefore highly desirable to ensure such iterative processes are privacy-preserving. We adopt the $\varepsilon$-differential privacy to measure such privacy guarantee; it is generally achieved by perturbing the algorithm such that the probability distribution of its output is relatively insensitive to any change to a single record in the input \cite{Dwork2006}.

Existing approaches to decentralizing the above problem primarily consist of subgradient-based algorithms \cite{nedic2009,lobel2011,Gade} and ADMM-based algorithms \cite{xu2017adaptive,xu2016,wei2012,ling2014,shi2014,zhang2014,ling2016}. It has been shown that ADMM-based algorithms can converge at the rate of $O(\frac{1}{k})$ while subgradient-based algorithms typically converge at the rate of $O(\frac{1}{\sqrt{k}})$, where $k$ is the number of iterations \cite{wei2012}. In this study, we will solely focus on ADMM-based algorithms. While a number of differentially private (sub)gradient-based distributed algorithms have been proposed \cite{hale2015,huang2015,han2017,bellet2017}, the same is much harder for ADMM-based algorithms due to its computational complexity stemming from the fact that each node is required to solve an optimization problem in each iteration. Differentially private ADMM  has been studied in   \cite{zhang2017,xueru,DP_ADMM}. In particular, \rev{Zhang and Zhu \cite{zhang2017} }proposes the dual/primal variable perturbation method to inspect the privacy loss of one node in every single iteration; this, however, is not sufficient for guaranteeing privacy as an adversary can potentially use the revealed results from all iterations to perform inference. \rev{Zhang et al. \cite{xueru} }addresses this issue by inspecting the total privacy loss over the entire process and the entire network; A penalty perturbation method is proposed which may improve the privacy-accuracy tradeoff significantly. \rev{Huang et al. \cite{DP_ADMM} }\respthree{R3.5}  applies the first-order approximation to the augmented Lagrangian in all iterations; however, this method requires a central server to average all updated primal variables over the network in each iteration. 

Since privacy leakage accumulates over iterations, the total privacy loss over the entire process can be substantial, making it difficult to balance the privacy-accuracy tradeoff. In our prior work \cite{xueru} we introduced a penalty perturbation method to achieve a better tradeoff.  While the method shows significant improvement with the right choice of penalty parameters, this improvement is heavily dependent on such choices and is not guaranteed.  It is therefore important to seek guaranteed improvement in the privacy-accuracy tradeoff for ADMM-based algorithms, which is the subject of the present paper.\resptwo{R2.1}

In this study, we present Recycled ADMM (R-ADMM),
a modified version of ADMM where the privacy leakage only happens during half of the updates (Algorithm \ref{A0}). Specifically, we adopt a linearized approximated optimization in every even iteration, whose solution is calculated directly using results from the previous, odd iteration; this solution is also used for updating the primal variable. These approximated updates incur no privacy loss and require much less computation. {Compared with conventional ADMM, R-ADMM requires much less perturbation to provide the same level of privacy protection, thereby improving the privacy-accuracy trade-off.}\respone{R1.1}

We then further generalize R-ADMM and present a modified R-ADMM, referred to as MR-ADMM, which employs ideas proposed in \cite{xueru} and can accommodate non-constant penalty parameters which are also entity's own private information (Algorithm \ref{A1}). {Since the penalty parameter controls the updating step size, the algorithm can be more robust by decreasing the step size. It allows the algorithm to tolerate more noise, i.e., be more private, without jeopardizing too much accuracy.  As a result the privacy-accuracy trade-off is further improved.}\respone{R1.1}
 
 Both of these algorithms are essentially modifications of the original distributed ADMM algorithm; privacy in these algorithms are provided by introducing noise.  Accordingly, the private versions of these algorithms are developed using the objective perturbation method \cite{chaudhuri2011} (Algorithm \ref{A2}). We establish a sufficient condition for the convergence of both algorithms and characterize their corresponding total privacy loss for private algorithms. Both analysis and experiments on real-world datasets show that as compared with conventional ADMM algorithm, R-ADMM can improve the privacy-accuracy tradeoff significantly with much less computation. Moreover, by controlling the penalty parameters in MR-ADMM, this privacy-accuracy tradeoff is further improved.     

The remainder of the paper is organized as follows. We present problem formulation and the definition of differential privacy and ADMM in Section \ref{sec:pre}. Three algorithms are introduced in Section \ref{sec:radmm} including R-ADMM, MR-ADMM and the private MR-ADMM. The convergence analysis of non-private MR-ADMM, privacy analysis and generalization performance analysis of (non)-private MR-ADMM are presented in Section \ref{sec:mradmm}, \ref{sec:pradmm} and \ref{sec:sample}, respectively. Discussion is given in Section \ref{sec:discussion}. Numerical results are illustrated in Section \ref{sec:numerical} and Section \ref{sec:conclusion} concludes the paper. 

\section{Preliminaries }\label{sec:pre}
\subsection{Problem Formulation}
Consider a connected network\footnote{A connected network is one in which every node is reachable (via a path) from every other node.} given by an undirected graph $G(\mathscr{N},\mathscr{E})$, which consists of a set of nodes $\mathscr{N} = \{1,2,\cdots,N\}$ and a set of edges $\mathscr{E} = \{1,2,\cdots,E\}$. Two nodes can exchange information if and only if they are connected by an edge. Let $\mathscr{V}_i$ denote node $i$'s set of neighbors, excluding itself. Let $D_i$ be node $i$'s dataset. 

\rev{Consider an optimization problem over this network of $N$ nodes, where the overall objective function can be decomposed into $N$ sub-objective functions and each depends on a node's local dataset, i.e., 
\begin{equation}\label{eq:prelimi_1}
\min_{f_c}\text{Obj}(f_{c},D_{all}) = \sum_{i=1}^{N}O(f_c,D_i)
\end{equation}
The goal is to find a (centralized) optimal solution $f_c\in \mathbb{R}^d$ over the union of all local datasets $D_{all} = \cup_{i \in \mathscr{N}} D_i$ in a distributed manner using ADMM, while providing privacy guarantee for each data sample.}\resptwo{R2.2}\respthree{R3.3}

\subsection{Differential Privacy \cite{Dwork2006}}
A randomized algorithm $ \mathscr{A}(\cdot)$ taking a dataset as input satisfies $\varepsilon$-differential privacy if for any two datasets $D$, $\hat{D}$ differing in at most one data point, and for any set of possible outputs $S \subseteq \text{range}(\mathscr{A})$, $\text{Pr}(\mathscr{A}(D) \in S)\leq e^{\varepsilon}\text{Pr}(\mathscr{A}(\hat{D}) \in S)$ holds.
We call two datasets differing in at most one data point as neighboring datasets. \rev{$\varepsilon\in[0,\infty)$ can be used to quantify the privacy loss/guarantee}. The above definition suggests that for a sufficiently small $\varepsilon$, an adversary will observe almost the same output regardless of the presence (or value change) of any one individual in the dataset; this is what provides privacy protection for that individual, \rev{the smaller $\varepsilon$, the smaller privacy loss, the stronger privacy guarantee.} 

\rev{Differential privacy is a worse-case measure; i.e., the bound
is over all possible random outputs and all possible inputs. It is
a strong guarantee, as it can protect against attackers with any
side information. Moreover, it is immune to post-processing \cite{dwork2014algorithmic};
i.e., given only the differentially private output without additional
information about the true data, it is impossible for
attackers to make it less differentially private.  }\respthree{R3.1}\respthree{R3.4}

\rev{For an optimization problem over a dataset, there are many approaches to randomizing the output to preserve differential privacy and some of the most commonly used are as follows. (1) Output perturbation: solve the optimization problem first and then add zero-mean noise (e.g., Laplace, Gaussian) to the optimal solution. (2) Objective perturbation: add a noisy term to the objective function first and then solve the perturbed optimization problem. Because of this randomness, the accuracy of the output also decreases accordingly. The more perturbation, the output will be less accurate but it also provides the stronger privacy for individuals. Therefore, there is a privacy-accuracy tradeoff, and an important issue is how to improve this tradeoff so that the output can be more accurate under the same privacy guarantee. }\respthree{R3.1}
\subsection{Conventional ADMM}
To decentralize \eqref{eq:prelimi_1}, let $f_i$ be the local classifier of each node $i$. To achieve consensus, i.e., $f_1 = f_2 = \cdots = f_N$, a set of auxiliary variables $\{w_{ij} | i \in \mathscr{N}, j \in \mathscr{V}_i\}$ are introduced for every pair of connected nodes.  As a result, \eqref{eq:prelimi_1} is reformulated equivalently as: 
 \begin{equation}\label{eq:prelimi_2}
\begin{aligned}
& \min_{\{f_i\},\{w_{ij}\}} 
& &\widetilde{\text{Obj}}(\{f_i\}_{i=1}^N,D_{all})  = \sum_{i=1}^{N}O(f_i,D_i) \\
&\text{   s.t.} 
& & f_i = w_{ij}, w_{ij} = f_j, \ \ \ i \in \mathscr{N}, j \in \mathscr{V}_i
\end{aligned}
\end{equation}

\rev{Let $\{f_i\}$ and $\{w_{ij}\}$ be the shorthand for $\{f_i\}_{i\in \mathscr{N}}$ and $\{w_{ij}\}_{i\in \mathscr{N}, j\in \mathscr{V}_i}$, respectively. }\respthree{R3.6}  
Let $\{w_{ij},\lambda_{ij}^k\}$ be the shorthand for $\{w_{ij},\lambda_{ij}^k\}_{i \in \mathscr{N},j \in \mathscr{V}_i, k \in \{a,b\}}$, where $\lambda_{ij}^a$, $\lambda_{ij}^b$ are dual variables corresponding to equality constraints $f_i = w_{ij}$ and $w_{ij}=f_j$ respectively. The objective in \eqref{eq:prelimi_2} can be solved using ADMM with the augmented Lagrangian: 
\begin{eqnarray}\label{eq:prelimi_3}
L_\eta (\{f_i\},\{w_{ij},\lambda_{ij}^k\}) = \sum_{i=1}^NO(f_i,D_i)\nonumber\\ +  \sum_{i=1}^N\sum_{j \in \mathscr{V}_i}(\lambda_{ij}^a)^T(f_i-w_{ij})+\sum_{i=1}^N\sum_{j \in \mathscr{V}_i}(\lambda_{ij}^b)^T(w_{ij}-f_j)\\
+ \sum_{i=1}^N\sum_{j \in \mathscr{V}_i}\dfrac{\eta}{2} (||f_i-w_{ij}||_2^2+||w_{ij}-f_j||_2^2)~.\nonumber
\end{eqnarray}
\rev{where $\eta$ is called the penalty parameter. }\respthree{3.7}In the $(t+1)$-th iteration, the ADMM updates consist of the following:
\begin{eqnarray}
f_i(t+1) &=& \underset{f_i}{\text{argmin}}\ L_\eta (\{f_i\},\{w_{ij}(t),\lambda_{ij}^k(t)\});\label{eq:prelimi_4}\\
w_{ij}(t+1) &=& \underset{w_{ij}}{\text{argmin}}\ L_\eta (\{f_i(t+1)\},\{w_{ij},\lambda_{ij}^k(t)\})\label{eq:prelimi_5};\\
\lambda_{ij}^a(t+1) &=& \lambda_{ij}^a(t) + \eta(f_i(t+1)-w_{ij}(t+1));\label{eq:prelimi_6}\\
\lambda_{ij}^b(t+1)& =& \lambda_{ij}^b(t) + \eta(w_{ij}(t+1)-f_j(t+1)).\label{eq:prelimi_7}
\end{eqnarray}
Using Lemma 3 in \cite{forero2010}, {if dual variables $\lambda_{ij}^a(t)$ and $\lambda_{ij}^b(t)$ are initialized to zero for all node pairs $(i,j)$, then $\lambda_{ij}^a(t) = \lambda_{ij}^b(t)$ and $\lambda_{ij}^k(t) = -\lambda_{ji}^k(t)$ will hold for all iterations with $k \in \{a,b\}, i \in \mathscr{N}, j \in \mathscr{V}_i$.} Let $\lambda_{i}(t) = \sum_{j \in \mathscr{V}_i}\lambda_{ij}^a(t) = \sum_{j \in \mathscr{V}_i}\lambda_{ij}^b(t)$, then the ADMM iterations \eqref{eq:prelimi_4}-\eqref{eq:prelimi_7} can be simplified as (Refer to Appendix A in \cite{xueru} for proof):
\begin{eqnarray}
f_i(t+1) &=& \underset{f_i}{\text{argmin}} \{ O(f_i,D_i) + 2\lambda_i(t)^Tf_i \nonumber \\
&&+  \eta \sum_{j \in \mathscr{V}_i}||\dfrac{1}{2}(f_i(t)+f_j(t))-f_i||_2^2~ \}~; \label{eq:prelimi_8} \\ 
\lambda_{i}(t+1) &=& \lambda_{i}(t) +  \dfrac{\eta}{2}\sum_{j \in \mathscr{V}_i}(f_i(t+1)-f_j(t+1))~. \label{eq:prelimi_9}
\end{eqnarray}
 \subsection{Private ADMM \cite{zhang2017} \& Private M-ADMM \cite{xueru}}
 In private ADMM \cite{zhang2017}, noise is added either to the updated primal variable before broadcasting to its neighbors (primal variable perturbation), or to the dual variable before updating its primal variable using \eqref{eq:prelimi_8} (dual variable perturbation). The privacy property is only evaluated for a single node and a single iteration, but neither method can effectively balance the privacy-accuracy tradeoff if the total privacy loss is considered. In our prior work \cite{xueru}, the total privacy loss of the whole network over the entire iterative process is considered. A modified ADMM (M-ADMM) was proposed to improve the privacy-accuracy tradeoff. Specifically, it explores \rev{the use of the penalty parameter $\eta$ }\respthree{R3.7} 
 in stabilizing the algorithm. M-ADMM allows each node to independently determine its penalty parameter and \rev{randomizes the objective function in primal update \eqref{eq:prelimi_8} by adding a linear noise term correlated to the penalty parameter }\respthree{R3.8}   
 while at the same time increasing the penalty over time.  By doing so it is shown that the privacy and accuracy can be improved simultaneously.

\section{Algorithms}\label{sec:radmm} 
\subsection{Recycled ADMM (R-ADMM)} 

\subsubsection{Main idea}
Fundamentally, the accumulation of privacy loss over iterations stems from the fact that the \rev{individual data $D_{all}$} is used in every primal update. If the updates can be made without \rev{directly using this original data}, but only from computational results that already exist, then the privacy loss originating from these updates will be zero, while at the same time the computational cost may be reduced significantly. \rev{This idea of ``recycling information" is also supported by the immunity to post-processing that differential privacy possesses \cite{dwork2014algorithmic}, i.e., any computation over an output that is already differentially private cannot incur additional privacy loss. Toward this end, R-ADMM 
modifies the ADMM algorithm such that we repeatedly use earlier computational results to make updates. }\respthree{R3.4}

\begin{algorithm}\label{A0}
	\textbf{Input: }{$\{D_i\}_{i=1}^N$}
	
	\textbf{Initialize: }$\forall i$, generate $f_i(0)$ randomly, $\lambda_i(0) = \textbf{0}_{d \times 1}$ 
	
	\For{$k=1$ \KwTo$K$}{
		\For{$i=1$ \KwTo$\mathscr{N}$}{
			Update primal variable $f_i(2k-1)$ via \eqref{eq:r-admm3};
			
			Calculate the gradient $\nabla O(f_i(2k-1),D_i)$;
			
			Broadcast $f_i(2k-1)$ to all neighbors $j \in \mathscr{V}_i$.
		}	
		
		\For{$i=1$ \KwTo$\mathscr{N}$}{
			Calculate  $\eta\sum_{j \in \mathscr{V}_i}(f_i(2k-1)-f_j(2k-1))$;
			
			Update dual variable $\lambda_i(2k-1)$ via \eqref{eq:r-admm4}.
		}	
		\For{$i=1$ \KwTo$\mathscr{N}$}{
			Use the stored $\nabla O(f_i(2k-1),D_i)$ and $\eta\sum_{j \in \mathscr{V}_i}(f_i(2k-1)-f_j(2k-1))$ to update primal variable $f_i(2k)$ via \eqref{eq:r-admm1};
			
			Keep the dual variable $\lambda_i(2k)=\lambda_i(2k-1)$;
			
			Broadcast $f_i(2k)$ to all neighbors $j \in \mathscr{V}_i$.
		}	
	}
	\textbf{Output: }{primal $\{f_i(2K)\}_{i=1}^N$ and dual $\{\lambda_i(2K)\}_{i=1}^N$}	\caption{Recycled-ADMM (R-ADMM)}
\end{algorithm}
\subsubsection{Making information recyclable}
ADMM can outperform gradient-based methods in terms of requiring fewer number of iterations for convergence; this however comes at the price of high computational cost in every iteration. In particular, the primal variable is updated by performing an optimization in each iteration. In \cite{mokhtari2015,ling2014,ling2015dlm}, either a linear or  quadratic approximation of the objective function is used to obtain an inexact solution in each iteration in lieu of solving the original optimization problem. While this clearly lowers the computational cost, the approximate computation is performed using the local, \rev{individual data} in every iteration, which means that privacy loss inevitably accumulates over the iterations. 

We begin by modifying ADMM in such a way that in every even iteration, without using  \rev{data $D_{all}$}, the primal variable is updated solely based on the existing computational results from the previous, odd iteration. Compared with conventional ADMM, these updates incur no privacy loss and less computation.  Since the computational results are repeatedly used, this method is referred to as Recycled ADMM (R-ADMM). 

Specifically, in the $2k$-th (even) iteration, $O(f_i,D_i)$ (Eqn. \eqref{eq:prelimi_8}, primal update optimization) is approximated by $O(f_i,D_i)\approx O(f_i(2k-1),D_i) + \nabla O(f_i(2k-1),D_i)^T(f_i-f_i(2k-1)) + \frac{\gamma}{2}||f_i-f_i(2k-1)||_2^2$ $(\gamma \geq 0)$ and only the primal variables are updated. Using the first-order condition, the updates in the $2k$-th iteration become:
\begin{eqnarray}
f_i(2k)=f_i(2k-1) - \frac{1}{2\eta V_i+\gamma}\{\nabla O(f_i(2k-1),D_i) \nonumber \\ +2\lambda_i(2k-1)+\eta\sum_{j\in \mathscr{V}_i}(f_i(2k-1)-f_j(2k-1))\}
\label{eq:r-admm1}~;\\
\lambda_{i}(2k) = \lambda_{i}(2k-1)~. \label{eq:r-admm2}
\end{eqnarray}
In the $(2k-1)$-th (odd) iteration, the updates are kept the same as \eqref{eq:prelimi_8}\eqref{eq:prelimi_9}:
\begin{eqnarray}
f_i(2k-1) = \underset{f_i}{\text{argmin}} \{ O(f_i,D_i) + 2\lambda_i(2k-2)^Tf_i \nonumber \\
+  \eta \sum_{j \in \mathscr{V}_i}||\dfrac{1}{2}(f_i(2k-2)+f_j(2k-2))-f_i||_2^2~ \}~;  \label{eq:r-admm3}\\ 
\lambda_{i}(2k-1) = \lambda_{i}(2k-2) \nonumber\\+  \dfrac{\eta}{2}\sum_{j \in \mathscr{V}_i}(f_i(2k-1)-f_j(2k-1))~. \label{eq:r-admm4}
\end{eqnarray}
Note that in the $(2k)$-th (even) iteration, we need the gradient $\nabla O(f_i(2k-1),D_i)$ and primal difference $\dfrac{\eta}{2}\sum_{j \in \mathscr{V}_i}(f_i(2k-1)-f_j(2k-1))$ for the updates; these are available directly from the previous, $(2k-1)$-th (odd) iteration,
i.e., this information can be recycled.  In this sense, R-ADMM may be viewed as alternating between conventional ADMM (odd iterations) and a variant of gradient descent (even iterations), where $\frac{1}{2\eta {V}_i + \gamma}$ is the step-size with a slightly modified gradient term. The complete procedure is shown in Algorithm \ref{A0}.

\begin{figure*}[t]
	\normalsize
	\begin{eqnarray}
	f_i(2k-1) &=& \underset{f_i}{\text{argmin}} \{ O(f_i,D_i) + 2\lambda_i(2k-2)^Tf_i 
	+  \eta_i(2k-1) \sum_{j \in \mathscr{V}_i}||\dfrac{1}{2}(f_i(2k-2)+f_j(2k-2))-f_i||_2^2~ \}~;  \label{eq:mr-admm3}\\ 
	\lambda_{i}(2k-1) &=& \lambda_{i}(2k-2)+  \dfrac{\eta_i(2k-1)}{2}\sum_{j \in \mathscr{V}_i}(f_i(2k-1)-f_j(2k-1))~. \label{eq:mr-admm4}\\
	f_i(2k)&=&f_i(2k-1) - \frac{1}{2\eta_i(2k-1) V_i+\gamma}\{\nabla O(f_i(2k-1),D_i) +2\lambda_i(2k-1)\nonumber\\&&+\eta_i(2k-1)\sum_{j\in \mathscr{V}_i}(f_i(2k-1)-f_j(2k-1))\}
	\label{eq:mr-admm1}~;\\
	\lambda_{i}(2k) &=& \lambda_{i}(2k-1)~. \label{eq:mr-admm2}
	\end{eqnarray}
	\hrulefill
\end{figure*}
\begin{algorithm}\label{A1}
	\textbf{Input: }{$\{D_i\}_{i=1}^N$}
	
	\textbf{Initialize: }$\forall i$, generate $f_i(0)$ randomly, $\lambda_i(0) = \textbf{0}_{d \times 1}$ 
	
	\textbf{Parameter: }$\forall i$, select $\{\eta_i(2k-1)\}_{k=1}^K$ s.t. $0<\eta_i(2k-1)\leq \eta_i(2k+1)<+\infty$, $\forall k$ 
	
	\For{$k=1$ \KwTo$K$}{
		\For{$i=1$ \KwTo$\mathscr{N}$}{
			Update primal variable $f_i(2k-1)$ via \eqref{eq:mr-admm3};
			
			Calculate the gradient $\nabla O(f_i(2k-1),D_i)$;
			
			Broadcast $f_i(2k-1)$ to all neighbors $j \in \mathscr{V}_i$.
		}	
		
		\For{$i=1$ \KwTo$\mathscr{N}$}{
			Calculate  $\eta_i(2k-1)\sum_{j \in \mathscr{V}_i}(f_i(2k-1)-f_j(2k-1))$;
			
			Update dual variable $\lambda_i(2k-1)$ via \eqref{eq:mr-admm4}.
		}	
		\For{$i=1$ \KwTo$\mathscr{N}$}{
			Use the stored $\nabla O(f_i(2k-1),D_i)$ and $\eta_i(2k-1)\sum_{j \in \mathscr{V}_i}(f_i(2k-1)-f_j(2k-1))$ to update primal variable $f_i(2k)$ via \eqref{eq:mr-admm1};
			
			Keep the dual variable $\lambda_i(2k)=\lambda_i(2k-1)$;
			
			Broadcast $f_i(2k)$ to all neighbors $j \in \mathscr{V}_i$.
		}	
	}
	\textbf{Output: }{primal $\{f_i(2K)\}_{i=1}^N$ and dual $\{\lambda_i(2K)\}_{i=1}^N$}	\caption{Modified R-ADMM (MR-ADMM)}
\end{algorithm}

\subsection{Modified R-ADMM (MR-ADMM)}
\subsubsection{Making $\eta$ a node's private information}
R-ADMM requires that the penalty parameter $\eta$ be fixed for all nodes in all iterations. Inspired by M-ADMM in \cite{xueru}, we modify R-ADMM such that each node can independently determine its penalty parameter in each iteration. Specifically, replace $\eta$ in \eqref{eq:r-admm1}, \eqref{eq:r-admm3} and \eqref{eq:r-admm4} with $\eta_i(2k-1)$. The updating formula is then given in \eqref{eq:mr-admm3}-\eqref{eq:mr-admm2}. The complete procedure is shown in Algorithm \ref{A1}.

\subsubsection{Relationship between R-ADMM and MR-ADMM} MR-ADMM is a generalized version of R-ADMM. If fix $\eta_i(2k-1) = \eta$, $\forall k$, then MR-ADMM  will be reduced to R-ADMM. 
\begin{algorithm}\label{A2}
	\textbf{Input: }{$\{D_i\}_{i=1}^N$}, $\{\alpha_i(1),\cdots, \alpha_i(K)\}_{i=1}^N$
	
	\textbf{Initialize: }$\forall i$, generate $f_i(0)$ randomly, $\lambda_i(0) = \textbf{0}_{d \times 1}$ 
	
	\textbf{Parameter: }$\forall i$, select $\{\eta_i(2k-1)\}_{k=1}^K$ s.t. $0<\eta_i(2k-1)\leq \eta_i(2k+1)<+\infty$, $\forall k$ and $\eta_i(1)$ satisfies $2c_1<\min_i\{\frac{B_i}{C}(\frac{\rho}{N} + 2\eta_i(1) V_i)\}$
	
	\For{$k=1$ \KwTo$K$}{
		\For{$i=1$ \KwTo$\mathscr{N}$}{
			Generate noise $\epsilon_i(2k-1) \sim \exp(-\alpha_i(k)||\epsilon||_2)$;
			
			Update primal variable $f_i(2k-1)$ via \eqref{eq:P_modify_2};
			
			
			Broadcast $f_i(2k-1)$ to all neighbors $j \in \mathscr{V}_i$.
		}	
		
		\For{$i=1$ \KwTo$\mathscr{N}$}{
			Calculate  $\eta_i(2k-1)\sum_{j \in \mathscr{V}_i}(f_i(2k-1)-f_j(2k-1))$;
			
			Update dual variable $\lambda_i(2k-1)$ via \eqref{eq:mr-admm4}.
		}	
		\For{$i=1$ \KwTo$\mathscr{N}$}{
			Use the stored information $\epsilon_i(2k-1)+\nabla O(f_i(2k-1),D_i)$ and $\eta_i(2k-1)\sum_{j \in \mathscr{V}_i}(f_i(2k-1)-f_j(2k-1))$ to update primal variable $f_i(2k)$ via \eqref{eq:P_modify_3};
			
			Keep the dual variable $\lambda_i(2k)=\lambda_i(2k-1)$;
			
			Broadcast $f_i(2k)$ to all neighbors $j \in \mathscr{V}_i$.
		}	
	}
	\textbf{Output: }{Upper bound of the total privacy loss $\beta$; primal $\{f_i(2K)\}_{i=1}^N$ and dual $\{\lambda_i(2K)\}_{i=1}^N$}	
	\caption{Private MR-ADMM}
\end{algorithm}

\subsubsection{Role of $\eta_i(2k-1)$ in stabilizing the algorithm}
\rev{The penalty parameter $\eta_i(2k-1)$ directly controls the step size of the algorithm. Since the goal is to minimize the objective in \eqref{eq:mr-admm3}, if $\eta_i(2k-1)$ is larger, the solution $f_i(2k-1)$ will be closer to the primal variable in the previous iteration so that the penalty term $\sum_{j \in \mathscr{V}_i}||\dfrac{1}{2}(f_i(2k-2)+f_j(2k-2))-f_i||_2^2$ will be small. In other words, larger $\eta_i(2k-1)$ results in smaller update of the primal variable $f_i(2k-1)$. In even updates \eqref{eq:mr-admm1}, $\frac{1}{2\eta_i(2k-1)V_i + \gamma}$ can also be regarded as step size as mentioned earlier. Therefore, increasing $\eta_i(2k-1)$ decreases the step size in both even and odd iterations. 
	
Without perturbation, a decreasing step size might slow down the convergence. However, when the algorithm is perturbed with added noise, a smaller step size could prevent the variable from deviating too much from the optimal solution in each update, which in turn stabilizes the algorithm. In the rest of paper, we will introduce a private algorithm by perturbing MR-ADMM and illustrate how we can use our ability to control stability via $\eta_i(2k-1)$ to improve the accuracy of algorithm without jeopardizing   privacy. 	
}\respone{R1.1}\respthree{R3.7}

\subsection{Private MR-ADMM}

In this section we present a privacy preserving version of MR-ADMM. Since MR-ADMM is a generalized version of R-ADMM, the private version of R-ADMM can be built in a similar way. In odd iterations, we adopt the objective perturbation \cite{chaudhuri2011} where a random linear term $\epsilon_i(2k-1)^Tf_i$ is added to the objective function in \eqref{eq:r-admm3}
\footnote{Pure differential privacy was adopted in this work, but the weaker $(\epsilon,\delta)$-differential privacy can be applied as well.}, where $\epsilon_i(2k-1)$ follows the probability density proportional to $\exp\{-\alpha_i(k)||\epsilon_i(2k-1)||_2\}$. \rev{Consequently the objective function for updating the primal variable $f_i(2k-1)$  becomes ${L}_i^{priv}(2k-1)$ given as follows:}\respthree{R3.9}
\begin{eqnarray}\label{eq:P_modify_1}
{L}_i^{priv}(2k-1) = O(f_i,D_i) + (2\lambda_i(2k-2)+\epsilon_i(2k-1))^Tf_i \nonumber \\
+  \eta_i(2k-1) \sum_{j \in \mathscr{V}_i}||\dfrac{1}{2}(f_i(2k-2)+f_j(2k-2))-f_i||_2^2 \nonumber
\end{eqnarray}
To generate this noisy vector $\epsilon_i(2k-1)$, choose the norm from the gamma distribution with shape $d$ and scale $\frac{1}{\alpha_i(k)}$ and the direction uniformly, where $d$ is the dimension of the feature space. Node $i$'s local result (primal variable) is obtained by finding the optimal solution to the private objective function: 
\begin{equation}\label{eq:P_modify_2}
f_i(2k-1) = \underset{f_i}{\text{argmin}}\ {L}_i^{priv}(2k-1) , \ \ i \in \mathscr{N}~. 
\end{equation}
In the $2k$-th iteration, use the stored results $\epsilon_i(2k-1) + \nabla O(f_i(2k-1),D_i)$ and $\eta_i(2k-1)\sum_{j \in \mathscr{V}_i}(f_i(2k-1)-f_j(2k-1))$ to update primal variables, where the latter can be obtained from the dual update in the $(2k-1)$-th update, and the former can be obtained directly from the KKT condition in the $(2k-1)$-th iteration:
\begin{eqnarray*}\label{info}
	\epsilon_i(2k-1) + \nabla O(f_i(2k-1),D_i) = -2\lambda_{i}(2k-2)\nonumber \\-\eta_i(2k-1)\sum_{j \in \mathscr{V}_i}(2f_i(2k-1))-f_i(2k-2)-f_j(2k-2)) ~. 
\end{eqnarray*}
Then the even update is given by: 
\begin{eqnarray}\label{eq:P_modify_3}
f_i(2k)=f_i(2k-1) - \frac{1}{2\eta_i(2k-1) V_i+\gamma}\{2\lambda_i(2k-1)\nonumber \\+\underbrace{\epsilon_i(2k-1) +\nabla O(f_i(2k-1),D_i)}_{\text{the existing result by KKT} } \nonumber \\+\underbrace{\eta_i(2k-1)\sum_{j\in \mathscr{V}_i}(f_i(2k-1)-f_j(2k-1))}_{\text{the existing result by the previous dual update}}\}~. 
\end{eqnarray}

Algorithm \ref{A2} shows the complete procedure, where the condition used to generate $\eta_i(1)$ helps to bound the worst-case privacy loss but is not necessary in guaranteeing convergence.

\begin{figure*}[t]
	\normalsize
	\begin{eqnarray}
	\hat{f}(2k) &=&\hat{f}(2k-1) - \tilde{D}(2k-1)^{-1}\{\nabla \hat{O}(\hat{f}(2k-1),D_{all})+2\Lambda(2k-1)
	+ W(2k-1)(D-A)\hat{f}(2k-1)\}~;\label{eq:matrix_1}\\
	2\Lambda(2k) &=& 2\Lambda(2k-1)~;\label{eq:matrix_2}\\
	\textbf{0}_{N\times d}&=&\nabla\hat{O}(\hat{f}(2k-1),D_{all}) + 2\Lambda(2k-2)+W(2k-1)( 2D\hat{f}(2k-1)
	- (D+A)\hat{f}(2k-2))~;\label{eq:matrix_3}\\
	2\Lambda(2k-1) &=& 2\Lambda(2k-2)+W(2k-1)(D-A)\hat{f}(2k-1)~.\label{eq:matrix_4}
	\end{eqnarray}
	\hrulefill
	\begin{eqnarray}
	\textbf{0}_{N\times d}&=&\nabla\hat{O}(\hat{f}(2k-1),D_{all}) +W(2k-1) (D+A) \tilde{D}(2k-3)^{-1}\nabla \hat{O}(\hat{f}(2k-3),D_{all})\nonumber\\&+&W(2k-1)(D+A)(\hat{f}(2k-1)-\hat{f}(2k-3))
	\nonumber\\&+& W(2k-1)(D+A)\tilde{D}(2k-3)^{-1} W(2k-3)(D-A)\hat{f}(2k-3)\nonumber\\ &+&2\Lambda(2k-1)+W(2k-1) (D+A) \tilde{D}(2k-3)^{-1}2\Lambda(2k-3)~;\label{eq:final1} \\
	2\Lambda(2k-1) &=& 2\Lambda(2k-3)+W(2k-1)(D-A)\hat{f}(2k-1)~.\label{eq:final2}
	\end{eqnarray}	
	\hrulefill
	\begin{eqnarray}
	\textbf{0}_{N\times d}&=&\nabla\hat{O}(\hat{f}(t+1),D_{all}) +W(t+1) (D+A) \tilde{D}(t)^{-1}\nabla \hat{O}(\hat{f}(t),D_{all})+W(t+1)(D+A)(\hat{f}(t+1)-\hat{f}(t))
	\nonumber\\&+& W(t+1)(D+A)\tilde{D}(t)^{-1} W(t)(D-A)\hat{f}(t)+2\Lambda(t+1)+W(t+1) (D+A) \tilde{D}(t)^{-1}2\Lambda(t)~;\label{eq:c_2} \\
	2\Lambda(t+1) &=& 2\Lambda(t)+W(t+1)(D-A)\hat{f}(t+1)~.\label{eq:c_3}
	\end{eqnarray}	
	\hrulefill
\end{figure*}

\section{Convergence of non-private MR-ADMM\respone{R1.2}}\label{sec:mradmm}

Since MR-ADMM is a generalized version of R-ADMM, we focus on the convergence analysis of MR-ADMM in this section while the results immediately apply to R-ADMM by fixing $\eta_i(2k-1) = \eta$, $\forall k$.

We next show that the MR-ADMM (Eqn. \eqref{eq:mr-admm3}-\eqref{eq:mr-admm2}) converges to the optimal solution under a set of common technical assumptions. 

\textbf{\textit{Assumption 1}:} Function $O(f_i,D_i)$ is convex and differentiable in $f_i$, $\forall i$.

\textbf{\textit{Assumption 2}:} The solution set to the original problem \eqref{eq:prelimi_1} is nonempty and there exists at least one bounded element. 

\textbf{\textit{Assumption 3:}} For all $i \in \mathscr{N}$, $O(f_i,D_i)$ has Lipschitz continuous gradients, i.e., for any $f_i^1$ and ${f}_i^2$, we have: 
\begin{equation}\label{eq:assume1}
||\nabla O(f_i^1,D_i)-\nabla O(f_i^2,D_i)||_2 \leq M_i||f_i^1- f_i^2||_2
\end{equation}

By the KKT condition of the primal update \eqref{eq:mr-admm3}:
\begin{eqnarray}\label{eq:c_1}
0 = \nabla O(f_i(2k-1),D_i) + 2\lambda_i(2k-2)\nonumber+ \eta_i(2k-1)\\ \cdot\sum_{j \in \mathscr{V}_i}(2{f}_i(2k-1)-({f}_i(2k-2)+{f}_j(2k-2)))~.
\end{eqnarray}

Define the adjacency matrix $A\in \mathbb{R}^{N \times N}$ as: 
$$a_{ij} = \begin{cases}
1, \ \ \text{ if node } i\text{ and node }j \text{ are connected } \\
0, \ \ \text{ otherwise }~. 
\end{cases}$$

Stack the variables $f_i(t)$, $\lambda_i(t)$ and $\nabla O(f_i(t),D_i)$ for $i \in \mathscr{N}$ into matrices, i.e.,
$$ \hat{f}(t) = \begin{bmatrix}
f_1(t)^T\\
f_2(t)^T\\
\vdots\\
f_N(t)^T
\end{bmatrix}\in \mathbb{R}^{N\times d} \text{ , \ \   }\Lambda(t) = \begin{bmatrix}
\lambda_1(t)^T\\
\lambda_2(t)^T\\
\vdots\\
\lambda_N(t)^T
\end{bmatrix}\in \mathbb{R}^{N\times d} $$ 
$$ \nabla \hat{O}(\hat{f}(t),D_{all}) = \begin{bmatrix}
\nabla O(f_1(t),D_1)^T\\
\nabla O(f_2(t),D_2)^T\\
\vdots\\
\nabla O(f_N(t),D_N)^T
\end{bmatrix}\in \mathbb{R}^{N\times d} $$

Let $V_i = | \mathscr{V}_i|$ be the number of neighbors of node $i$, and define the degree matrix $D = \textbf{diag}([V_1;V_2;\cdots;V_N])\in \mathbb{R}^{N \times N}$, the diagonal matrix $\tilde{D}(2k-1)$ with $\tilde{D}(2k-1)_{ii} = 2\eta_i(2k-1) V_i + \gamma$, and the weight matrix $W(2k-1) = \textbf{diag}([\eta_1(2k-1);\eta_2(2k-1);\cdots;\eta_N(2k-1)])\in \mathbb{R}^{N \times N}$. Then for each $k$, the matrix form of \eqref{eq:mr-admm1}\eqref{eq:mr-admm2}\eqref{eq:c_1}\eqref{eq:mr-admm4} are given in \eqref{eq:matrix_1}-\eqref{eq:matrix_4}:

Writing $\hat{f}(2k-2)$ and $\Lambda(2k-2)$ in \eqref{eq:matrix_3}\eqref{eq:matrix_4} as functions of $\hat{f}(2k-3)$, $\Lambda(2k-3)$ using \eqref{eq:matrix_1}\eqref{eq:matrix_2}, we obtain Eqn. \eqref{eq:final1}\eqref{eq:final2}.

The convergence of the MR-ADMM is proved by showing that the pair ($\hat{f}(2k-1)$, $\Lambda(2k-1)$) from odd iterations converges to the optimal solution. To simplify the notation, we will re-index every two consecutive odd iterations $2k-3$ and $2k-1$ using $t$ and $t+1$, it results in  Eqn. \eqref{eq:c_2}\eqref{eq:c_3}.

Note that $D-A$ is the laplacian and $D+A$ is the signless Laplacian matrix of the network, with the following properties if the network is connected: {(i)} $D\pm A \succeq 0$ is positive semi-definite; {(ii)} $\text{Null}(D-A) = c\textbf{1}$, i.e., every member in the null space of $D-A$ is a scalar multiple of \textbf{1} with \textbf{1} being the vector of all $1$'s \cite{Jonathan2007}. 
\begin{lemma}\label{Lemma:1}
	[\textbf{First-order Optimality Condition} \cite{ling2016}] Under Assumptions 1 and 2, the following two statements are equivalent:
	\begin{itemize}
		\item $\hat{f}^* = [(f_1^*)^T;(f_2^*)^T;\cdots;(f_N^*)^T] \in \mathbb{R}^{N\times d}$ is consensual, i.e., $f_1^*=f_2^*=\cdots=f_N^*=f_c^*$ where $f_c^*$ is the optimal solution to \eqref{eq:prelimi_1}.
		\item There exists a pair $(\hat{f}^*,\Lambda^*)$ with $2\Lambda^* = (D-A)X$ for some $X\in \mathbb{R}^{N\times d}$ such that
		\begin{eqnarray}
		\nabla \hat{O}(\hat{f}^*,D_{all})+2\Lambda^*=\textbf{0}_{N\times d} ~; 
		\label{eq:c_4}\\
		(D-A)\hat{f}^* = \textbf{0}_{N\times d}~. 
		\label{eq:c_5}
		\end{eqnarray}
	\end{itemize}
\end{lemma}
Lemma \ref{Lemma:1} shows that a pair $(\hat{f}^*,\Lambda^*)$ satisfying \eqref{eq:c_4}\eqref{eq:c_5} is equivalent to the optimal solution of our problem, hence the convergence of the MR-ADMM is proved by showing that $(\hat{f}(t),\Lambda(t))$ in \eqref{eq:c_2}\eqref{eq:c_3} converges to a pair $(\hat{f}^*,\Lambda^*)$ satisfying \eqref{eq:c_4}\eqref{eq:c_5}. 

\begin{theorem}\label{thmC1}[\textbf{Sufficient Condition}]
	Consider the modified ADMM defined by \eqref{eq:c_2}\eqref{eq:c_3}. Let $\{\hat{f}(t),\Lambda(t)\}$ be outputs in each iteration and $\{\hat{f}^*,\Lambda^*\}$ a pair satisfying \eqref{eq:c_4}\eqref{eq:c_5}. Denote $D_M = \textbf{diag}([M_1^2;M_2^2;\cdots;M_N^2])\in \mathbb{R}^{N \times N}$ with $0<M_i<+\infty$ as given in Assumption 3. If $\eta_i(t+1)\geq \eta_i(t)>0$ and $\eta_i(t)<+\infty$ hold and the following two conditions can also be satisfied for some constants $L>0$ and $\mu>1$:
	\begin{eqnarray*}
		&(i)&I+W(t+1)(D+A) \tilde{D}(t)^{-1}\nonumber\\&& \succ \frac{L\mu}{2\sigma_{\min}(\tilde{D}(t))}(W(t+1)(D-A))^{+}D_M ~;\label{eq:c_6}\\
		&(ii)&W(t+1)(D+A)\nonumber\\&&\succ W(t+1) (D+A) \tilde{D}(t)^{-1} \Big(W(t)(D-A)\nonumber \\&&+\frac{2}{L}W(t+1)(D+A)\Big) +\frac{L\mu}{2\sigma_{\min}(\tilde{D}(t))(\mu-1)}D_M ~.\label{eq:c_7}
	\end{eqnarray*}
	where $\sigma_{\min}(\tilde{D}(t)) = \underset{i}{\min}\{2\eta_i(t) V_i+\gamma\}$ is the smallest singular value of $\tilde{D}(t)$, then $(\hat{f}(t),\Lambda(t))$ converges to $(\hat{f}^*,\Lambda^*)$.
\end{theorem}
\begin{proof}
See Appendix \ref{App1}.
\end{proof}

By controlling $\gamma$ to be sufficiently large, $\tilde{D}(t)_{ii} = 2\eta_i(t)V_i + \gamma$ will be large and conditions \textit{(i)(ii)} can always be satisfied under some constants $L>0$ and $\mu>1$.
Note that the conditions \textit{(i)(ii)} are sufficient but not necessary, so in practice convergence may be attained under weaker settings. 

For R-ADMM, take $L=2$ and $\mu = 2$, then condition \textit{(i)(ii)} are reduced to: 
\begin{eqnarray*}
	&(iii)&I+\eta(D+A) \tilde{D}^{-1}\nonumber \succ \frac{2}{\eta\sigma_{\min}(\tilde{D})}((D-A))^{+}D_M ~;\label{eq:c_8}\\
	&(iv)&\eta(D+A)\succ2\eta(D+A) \tilde{D}^{-1}\eta D +\frac{2}{\sigma_{\min}(\tilde{D})}D_M ~.\label{eq:c_9}
\end{eqnarray*}
Again for a sufficiently large $\gamma\geq 0$, \textit{(iii)(iv)} can be easily satisfied.

\section{Privacy Analysis}\label{sec:pradmm} 
In this section, we characterize the total privacy loss of private MR-ADMM as presented in Algorithm \ref{A2}. Similar to the previous section, the results also apply to private R-ADMM by fixing $\eta_i(2k-1) = \eta$, $\forall k$.

\rev{As mentioned earlier, Zhang and Zhu \cite{zhang2017} only quantifies the privacy loss of a single node in a single iteration, i.e., $	\frac{\text{Pr}(f_i(t) \in S_i|D_{i})}{\text{Pr}(f_i(t) \in S_i|\hat{D}_{i})}\leq \exp(\alpha_i(t))$ holds $\forall t,i$, where $\alpha_i(t)$ is the bound on the privacy loss of node $i$ at iteration $t$. However, in a distributed and iterative setting, the ``output'' of the algorithm is not merely the end result, but includes all intermediate results generated and exchanged during the iterative process; an attacker can use all such intermediate results to perform inference. For this reason, we adopt the differential privacy definition proposed in \cite{xueru} as follows, which bounds the total privacy loss during the entire iterative process.}\respthree{R3.2}
\begin{definition}\label{Def}	
	Consider a connected network $G(\mathscr{N},\mathscr{E})$ with a set of nodes $\mathscr{N} = \{1,2,\cdots,N\}$. Let $f(t) = \{f_i(t)\}_{i=1}^N$ denote the information exchange of all nodes in the $t$-th iteration.
	A distributed algorithm is said to satisfy $\beta$-differential privacy during $T$ iterations if for any two datasets $D_{all} = \cup_i D_i$ and $\hat{D}_{all} = \cup_i \hat{D}_i$, differing in at most one data point, and for any set of possible outputs $S$ during $T$ iterations, the following holds:
	\begin{equation*}
	\frac{\text{Pr}(\{f(t)\}_{t=0}^T \in S|D_{all})}{\text{Pr}(\{f(t)\}_{t=0}^T \in S|\hat{D}_{all})} \leq \exp(\beta)
	\end{equation*}
\end{definition}

\rev{The analysis is focused on the regularized empirical risk minimization (ERM) problem for binary classification, while its generalization is discussed in Section \ref{sec:discussion}. Let node $i$'s dataset be $D_i = \{(x_{i}^n,y_{i}^n) | n = 1,2,\cdots,B_i \}$, where $x_{i}^n \in \mathbb{R}^d$ is the feature vector representing the $n$-th sample belonging to $i$, $y_{i}^n \in \{-1,1\}$ the corresponding label, and $B_i$ the size of $D_i$. Then the sub-objective function for each node $i$ is defined as follows: 
	 $$O(f_i,D_i) = \dfrac{C}{B_i}\sum_{n=1}^{B_i} {\mathscr{L}}(y_{i}^n f_i^Tx_{i}^n ) + \dfrac{\rho}{N} R(f_i)~,$$
where $C \leq B_i$ and $\rho>0$ are constant parameters of the algorithm, the loss function $\mathscr{L}(\cdot)$ measures the accuracy of the classifier, and the regularizer $R(\cdot)$ helps prevent overfitting.}\resptwo{R2.2}\respthree{R3.3}

For this binary classification problem, we now state another result on the privacy property of the private MR-ADMM (Algorithm \ref{A2}) using definition \ref{Def} above and additional assumptions on $\mathscr{L}(\cdot)$ and $R(\cdot)$ as follows. 

\textbf{\textit{Assumption 4}:} The loss function $\mathscr{L}$ is strictly convex and twice differentiable. $|\nabla\mathscr{L}| \leq 1$ and $0 <\mathscr{L}''\leq c_1$ with $c_1$ being a constant. 

\textbf{\textit{Assumption 5}:} The regularizer $R$ is 1-strongly convex and twice continuously differentiable. 
\begin{lemma}\label{lemmaP1}
	Consider the private MR-ADMM (Algorithm \ref{A2}), $\forall k=1,\cdots K$, assume the total privacy loss up to the $(2k-1)$-th iteration can be bounded by $\beta_{2k-1}$, then the total privacy loss up to the $2k$-th iteration can also be bounded by $\beta_{2k-1}$. In other words, given the private results in odd iterations, outputting private results in the even iterations does not release more information about the input data.  
\end{lemma}
\begin{proof}
See Appendix \ref{App_2}.
\end{proof}

\begin{theorem}\label{thmP}
	Normalize feature vectors in the training set such that $||x_{i}^n||_2\leq 1$ for all $i \in \mathscr{N}$ and $n$. Then the private MR-ADMM algorithm (Algorithm \ref{A2}) satisfies the $\beta$-differential privacy with 
	\begin{equation}
	\beta \geq \underset{i \in \mathscr{N}}{\max}\{\sum_{k=1}^{K}\frac{2C}{B_i}(\frac{1.4c_1}{(\frac{\rho}{N}+2\eta_i(2k-1) V_i)} + \alpha_i(k))\}~. 
	\end{equation}
\end{theorem}
\begin{proof}
See Appendix \ref{App_3}.
\end{proof}

\section{Sample complexity analysis\respone{R1.2}\resptwo{R2.3}}\label{sec:sample}

\rev{We next quantify the generalization performance of (non)-private MR-ADMM. The analysis is focused on the ERM problem defined in Section \ref{sec:pradmm} and we assume samples from each node $i$ are drawn i.i.d. from a fixed distribution $P$. The expected loss of node $i$ using classifier $f_i(t)$ at time $t$ is given as $\mathcal{L}(f_i(t)) = \mathbb{E}_{(X,Y)\sim P}(\mathscr{L}(Y f_i(t)^TX)) $.  Similar to the analysis in \cite{chaudhuri2011,zhang2017}, we introduce a reference classifier $f_{ref}$ with expected loss $\mathcal{L}(f_{ref})$ and evaluate the generalization performance using the number of samples ($B_i$) required at each node to achieve $\mathcal{L}(f_i(t)) \leq \mathcal{L}(f_{ref}) + \tau$ with high probability.    }

\subsection{Non-private MR-ADMM}
\rev{As shown in Section \ref{sec:mradmm}, the sequence of outputs $\{f_i^{non}(2k-1)\}$ from odd iterations in non-private MR-ADMM converges to $f_i^* = f_c^*$ as $k\rightarrow\infty$. Therefore, there exists a constant $\Delta_i(k)$ for each node $i$ at the $(2k-1)$-th iteration such that $\mathcal{L}(f_i^{non}(2k-1))\leq \mathcal{L}(f_c^*) + \Delta_i(k)$. Using the same method as \cite{chaudhuri2011,zhang2017}, we have the following result.

	\begin{theorem}\label{thm:sample}
		Consider a regularized ERM problem with regularizer $R(f) = \frac{1}{2}||f||^2$ and let $f_{ref}$ be a reference classifier for all nodes and $\{f_i^{non}(2k-1)\}$ be a sequence of outputs of non-private MR-ADMM in odd iterations (Eqn. \eqref{eq:mr-admm3}). If the number of samples at node $i$ satisfies
		\begin{eqnarray*}\label{eq:sample}
B_i\geq w \max_k\{ \frac{||f_{ref}||^2\log(1/\delta)}{(\tau - \Delta_i(k))^2}\}
		\end{eqnarray*}
		for some constant $w$, then $f_i^{non}(2k-1)$ satisfies $$Pr(\mathcal{L}(f_i^{non}(2k-1))\leq \mathcal{L}(f_{ref})+\tau)\geq 1-\delta$$ 
		where $\tau>\Delta_i(k)$, $\forall i, k\in \mathbb{Z}_+$.
	\end{theorem} 
\begin{proof}
See Appendix \ref{App_4}.
\end{proof}

As expected, the number of required samples depends on the choice of the reference classifier via its $l_2$ norm $||f_{ref}||^2$, by imposing an upper bound $b_{ref}$ on $||f_{ref}||^2$.  The result shows that if $B_i$ satisfies $B_i\geq w \max_k\{ \frac{b_{ref}\log(1/\delta)}{(\tau - \Delta_i(k))^2}\}$, then the non-private intermediate classifier of each node at odd iterations will have an additional error no more than $\tau$ as compared to any classifier with $||f_{ref}||^2\leq b_{ref}$.
}

\subsection{private MR-ADMM}
\rev{We next present the result on the sample complexity of the private MR-ADMM algorithm. Similar to the analysis of non-private MR-ADMM, we bound the error of the intermediate classifier of each node at odd iterations. Since the algorithm is perturbed with different random noise in different iterations, to better analyze the effect of noise in a single iteration, we adopt a strategy similar to that used in \cite{zhang2017}, by intentionally fixing the noise in iterations after the targeted iteration. Specifically, $\forall i$, to compare the private $f_i^{priv}(2k-1)$ at the $(2k-1)$-th iteration with reference classifier $f_{ref}$, we slightly modify Algorithm \ref{A2} such that $\forall k'>k$, the added noise is fixed at $\epsilon_i(2k'-1) = \epsilon_i(2k-1)$, which allows us to solely study the effect of $\epsilon_i(2k-1)$. This problem can be formulated as a new MR-ADMM optimization problem where node $i$'s sub-objective function becomes $O^{new}(f_i,D_i) = O(f_i,D_i)+\epsilon_i(2k-1)^Tf_i$ and the initialization given by  $f_i(0) = f_i(2k-1)$, $\lambda_i(0) = \lambda_i(2k-1)$. Let $\{f_i^{new}(2k-1)\}$ be a sequence of outputs from odd iterations of this new algorithm; it converges to a fixed point $f_{new}^*$ as $k\rightarrow\infty$.  Therefore, there exists a constant $\Delta_i^{new}(k)$ for each node $i$ at the $(2k-1)$-th iteration such that $\mathcal{L}(f_i^{new}(2k-1))\leq \mathcal{L}(f_{new}^*) + \Delta_i^{new}(k)$.  Using this, we have the following result.  
\begin{theorem}\label{thm:sample_priv}
Consider a regularized ERM problem with regularizer $R(f) = \frac{1}{2}||f||^2$, let $f_{ref}$ be a reference classifier for all nodes and $\{f_i^{priv}(2k-1)\}$ be a sequence of outputs of private MR-ADMM in odd iterations. If the number of samples at node $i$ satisfies
\begin{eqnarray*}\label{eq:sample_priv}
B_i\geq w\max_k\{\frac{CN\log(1/\delta)}{\frac{NC(\tau-\Delta_i^{new}(k))^2}{2||f_{ref}||^2}-(1+a) \frac{Nd^2}{C(\alpha_i(k))^2}(\log(d/\delta))^2}\}
\end{eqnarray*}
for some constants $w$ and $a>0$, then $f_i^{priv}(2k-1)$ satisfies $$Pr(\mathcal{L}(f_i^{priv}(2k-1))\leq \mathcal{L}(f_{ref})+\tau)\geq 1-2\delta$$ 
where $\tau>\Delta_i^{new}(k)$, $\forall i, k\in \mathbb{Z}_+$.
\end{theorem}	
\begin{proof}
See Appendix \ref{App_5}.
\end{proof}
Compared to Theorem \ref{thm:sample}, we see an additional term imposed by the privacy constraints, i.e., $(1+a) \frac{Nd^2}{C(\alpha_i(k))^2}(\log(d/\delta))^2$. If $\alpha_i(k)\rightarrow\infty$, the result reduces to $B_i\geq w\max_k\{\frac{2||f_{ref}||^2\log(1/\delta)}{(\tau-\Delta_i^{new}(k))^2}\}$, the same as given in Theorem \ref{thm:sample}. The additional term shows that the higher dimension of features, the more injected noise, which would require more samples to achieve the same accuracy.  
}

\section{Discussion}\label{sec:discussion}

\subsection{Improving privacy-accuracy tradeoff}

\rev{We now provide some intuitive explanation as to why the ideas presented in this paper work.  We explored two key ideas to improve the privacy-accuracy tradeoff of a differentially private algorithm.  The first is to accomplish the computational task by repeatedly using the already released differentially private outputs. Utilizing differential privacy's immunity to post-processing, this information recycling incurs no additional privacy loss. Since less information is revealed during computation, less perturbation is required to obtain the same privacy guarantee, which then improves the privacy-accuracy tradeoff.  The second idea is to improve the the stability/robustness of the algorithm by directly controlling the penalty parameter. This allows the algorithm to accommodate more noise to improve privacy without sacrificing too much accuracy, which improves the privacy-accuracy tradeoff. }\respone{R1.1} 


\subsection{Other perturbation methods and privacy analysis tools}

\rev{While we have primarily used objective perturbation to make an algorithm differentially private and to calculate the privacy loss, it should be noted that this is done as an example to illustrate how MR-ADMM can outperform both R-ADMM and ADMM in the privacy-accuracy tradeoff.  Other perturbation methods such as output perturbation to achieve differential privacy (each node perturbs its primal variable before broadcasting to its neighbors) can be used as well; our conclusion would still hold. This is because our key ideas (revealing less information and making the algorithm more robust/stable to noise via the penalty parameter) are orthogonal to the choice of the perturbation method.}\respone{R1.3} 


\rev{Similarly, in our privacy analysis we have adopted the notion of pure $\varepsilon$-differential privacy to measure privacy. As a result, the bound on the total privacy loss can be fairly large.  It is also possible to adopt a weaker notion, the $(\varepsilon,\delta)$-differential privacy, to find a tighter bound on privacy loss by allowing the algorithm to violate $\varepsilon$-differential privacy with a small probability $\delta$. In this case, the total privacy loss can be calculated using more advanced composition theorems such as moments accountant \cite{abadi2016deep} and zero-concentrated differential privacy \cite{bun2016concentrated}. However, our key ideas (revealing less information and making the algorithm more robust/stable to noise via the penalty parameter) are orthogonal to the choice of the privacy definition and analysis tools used; thus the algorithmic properties will not be affected by such choices and the conclusion remains valid.}\resptwo{R2.4}\respone{R1.6}

\subsection{Privacy analysis for a broader class of optimizations}

\rev{In Section \ref{sec:pradmm}, the privacy property of the private MR-ADMM is analyzed for the ERM binary classification problem. This is so that we can easily compare with ADMM and M-ADMM in \cite{zhang2017,xueru}. 
%
This privacy analysis can be extended to more general forms of $O(f_i,D_i)$, such as multi-class settings. There have been extensive studies on the differentially private ERM with convex loss function \cite{wang2017differentially}, which can also be adopted for our framework.}\resptwo{R2.2}\respthree{R3.3}

\begin{figure}[h]
	\centering   
	\subfigure[R-ADMM: $\eta = 0.5$]{\label{fig1:a}\includegraphics[trim={2.7cm 0 3cm 0cm},clip=true, width=0.49\textwidth]{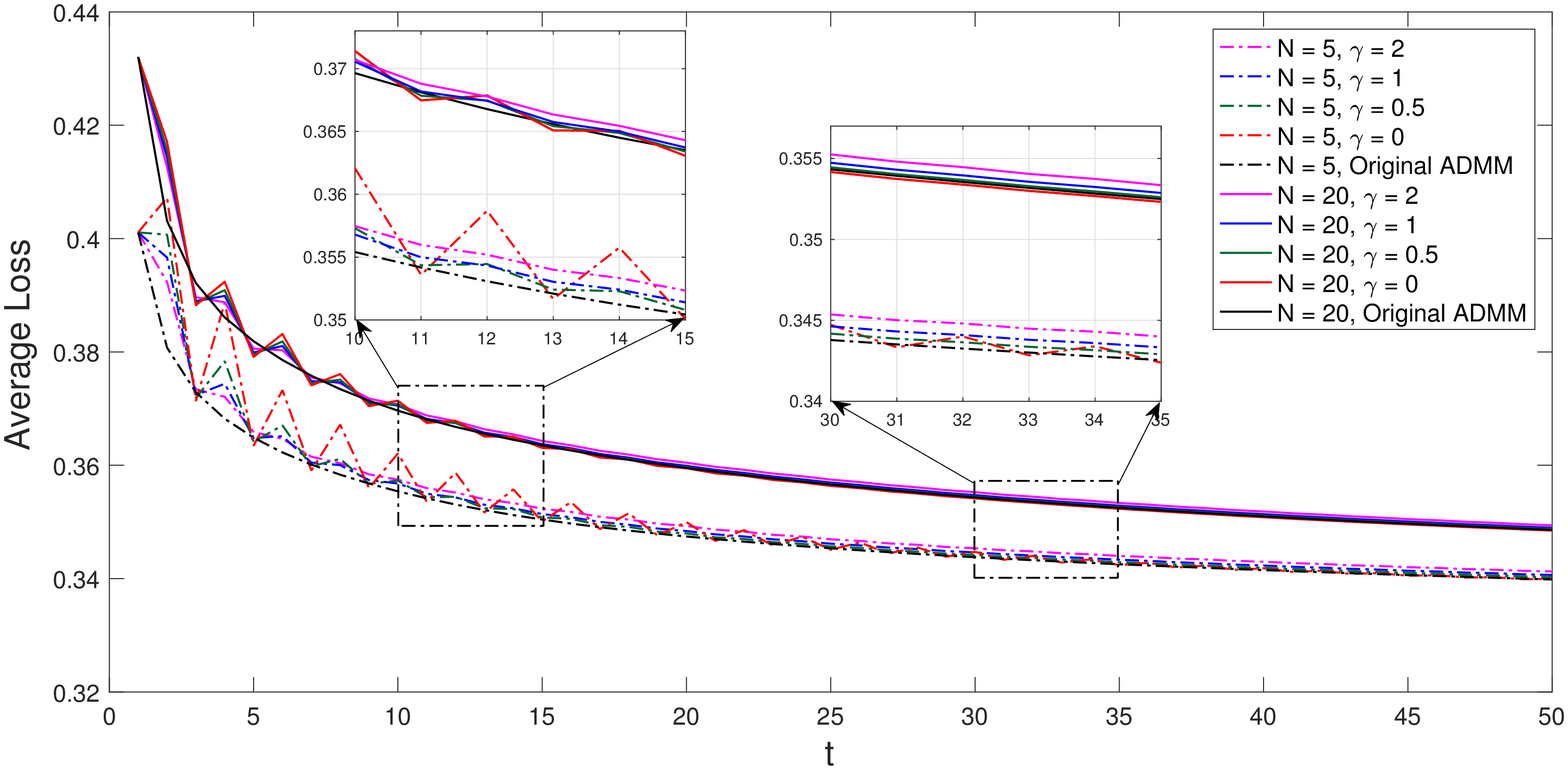}}
	\subfigure[$\eta_i(2k-1) = \hat{\eta}_iq_1(i)^k$]{\label{fig1:c}\includegraphics[trim={0.2cm 0 1.9cm 0cm},clip=true,width=0.24\textwidth]{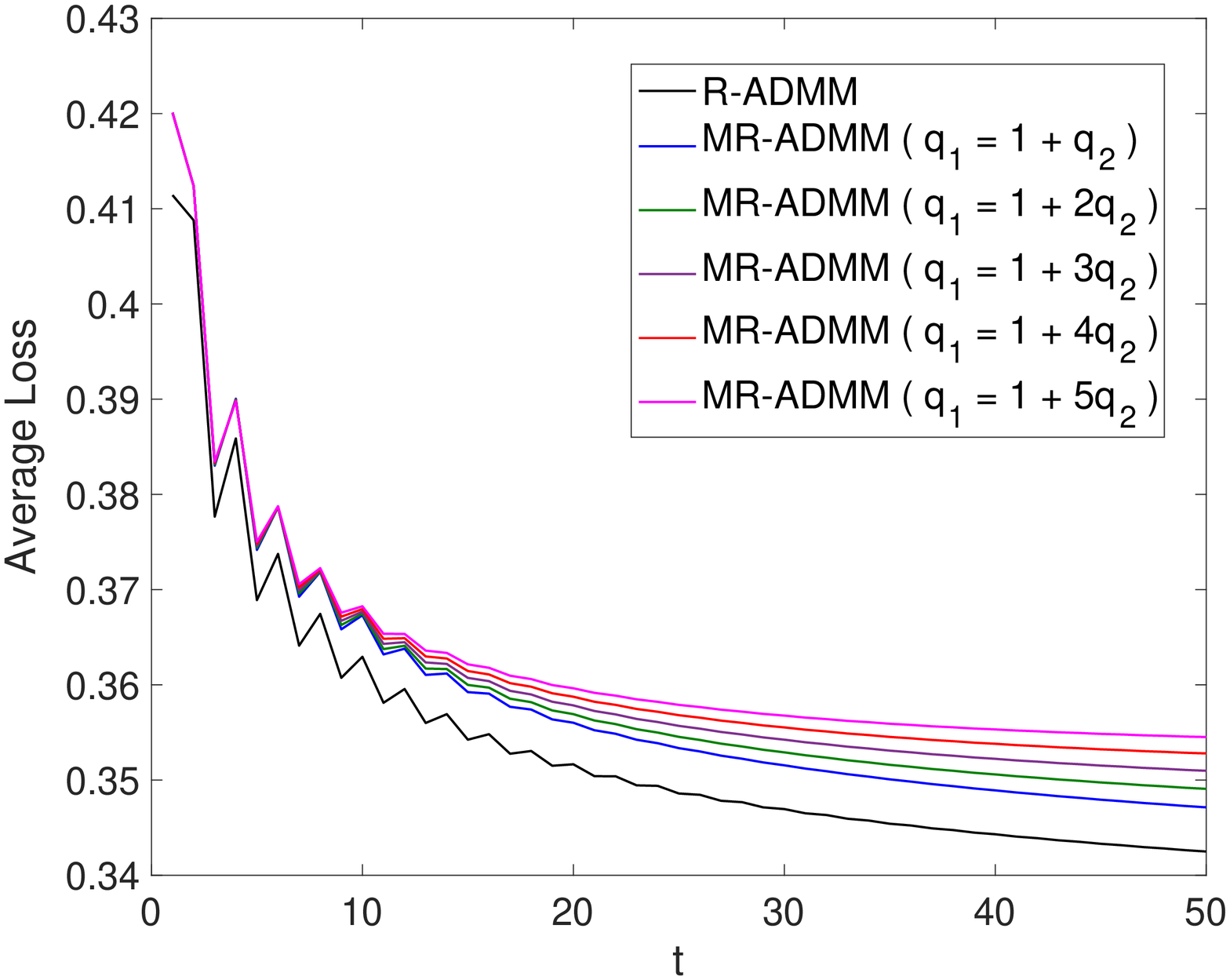}}
	\subfigure[$\eta_i(2k-1) = q_1^k$]{\label{fig1:b}\includegraphics[trim={0.2cm 0 1.9cm 0cm},clip=true,width=0.24\textwidth]{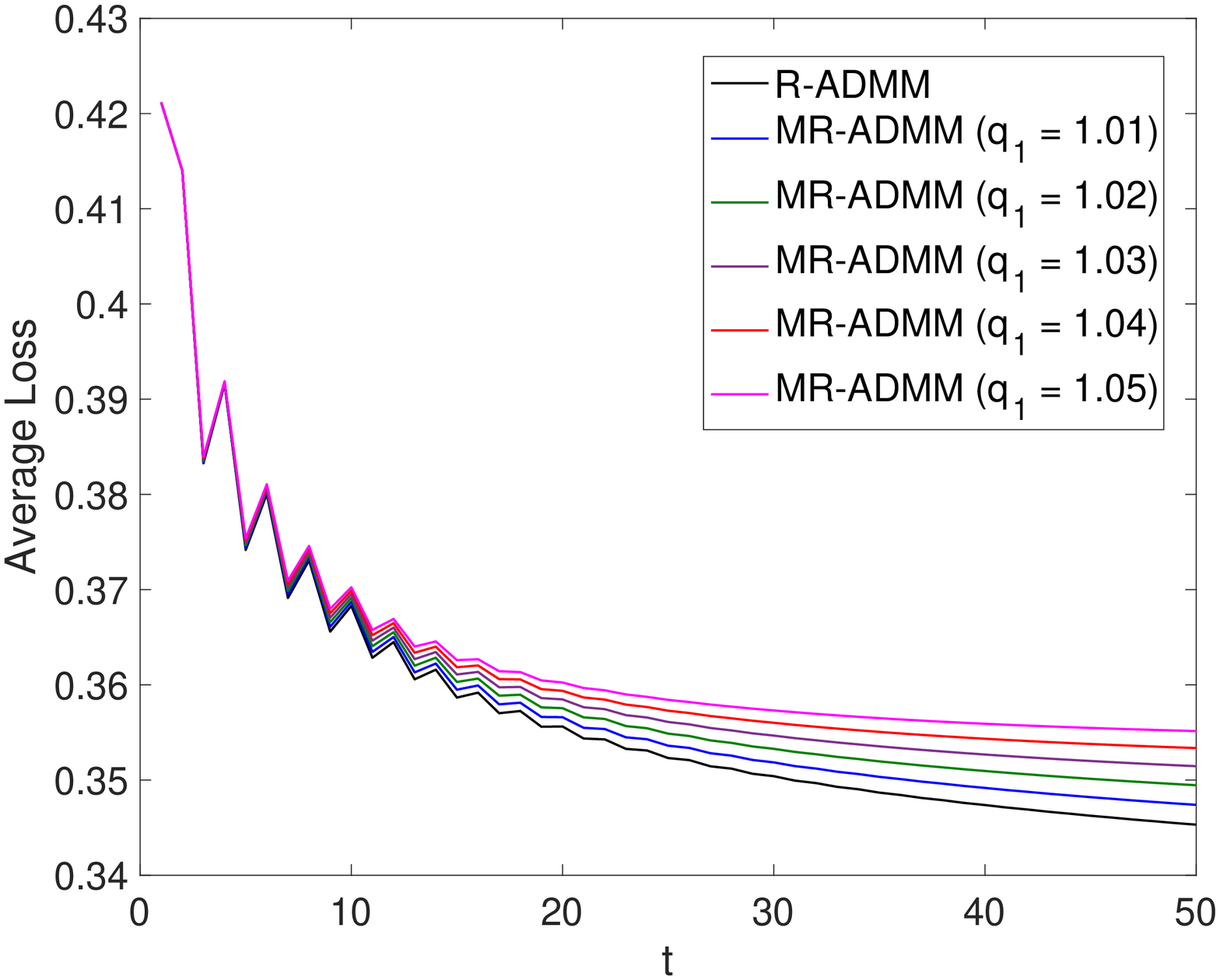}}
	\caption{Convergence properties of R-ADMM and MR-ADMM: Fig. \ref{fig1:a} illustrates the average loss over iterations of R-ADMM for the network of different sizes under fixed $\eta=0.5$ and different $\gamma$. Dashed (resp. solid) curves represent the performance over a randomly generated small (resp. large) network with $N=5$ (resp. $N=20$) nodes. 
		Fig. \ref{fig1:c}\ref{fig1:b} illustrate the average loss over iterations of MR-ADMM for a randomly generated network with $N=5$ nodes. Black curve represents the R-ADMM where $\eta_i(t) = \eta = 1$ is fixed for all nodes and all iterations. Each colored curve represents MR-ADMM with $\eta_i(2k-1)$ increasing over iterations at different speed. In Fig. \ref{fig1:c}, each node $i$ adopts $\eta_i(2k-1) = \eta_iq_1(i)^k$ as penalty parameter in $2k-1$-th iteration, where $[\eta_1,\cdots,\eta_5]= [1,1.03,1.02,0.8,1.01]$, $q_1 = [q_1(1),\cdots,q_1(5)] = \textbf{1}+kq_2$ (each $k\in \{1,\cdots, 5\}$ corresponds to one curve in plot) and $q_2 = [q_2(1),\cdots,q_2(5)] = [0.01,0.005,0.003,0.015,0.01]$. In Fig. \ref{fig1:b}, each node adopts the same penalty parameter $\eta_i(2k-1) = q_1^k$ in odd iterations.    }
	\label{fig1}
\end{figure}
\begin{figure}[h]
	\centering   
	\includegraphics[trim={0cm 0 0cm 0cm},clip=false, width=0.35\textwidth]{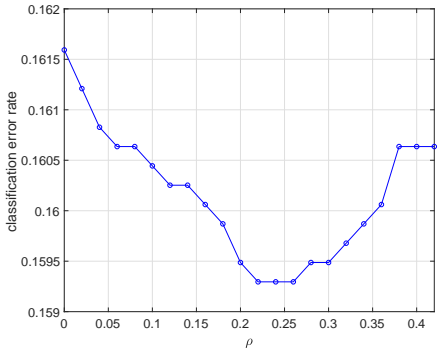}
	\caption{The effect of $\rho$, fixing $C = 1750$.}
	\label{fig:rho}
\end{figure}

\begin{figure}[h]
	\centering   
	\subfigure[Accuracy comparison for different $\gamma$ ($\alpha=1$)]{\label{fig:gamma1}\includegraphics[trim={1cm 0 1cm 0cm},clip=true, width=0.5\textwidth]{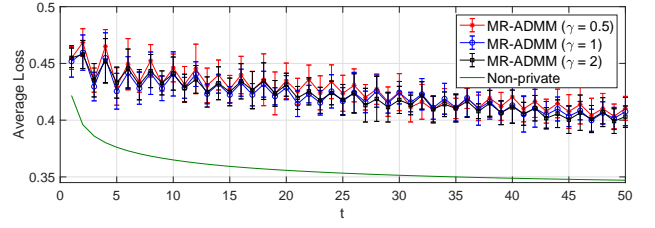}}
	\subfigure[Accuracy comparison for different $\gamma$ ($\alpha=2$)]{\label{fig:gamma2}\includegraphics[trim={1cm 0 1cm 0cm},clip=false, width=0.5\textwidth]{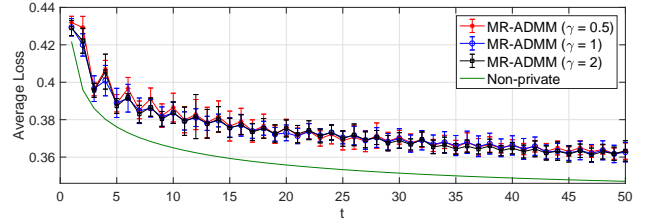}}
	\subfigure[Classification error rate comparison]{\label{fig:gamma3}\includegraphics[trim={0cm 2cm 0cm 0.6cm},clip=true, width=0.35\textwidth]{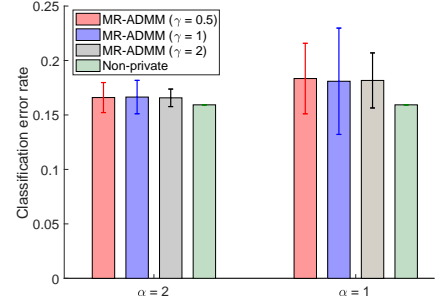}}
	\caption{\rev{The effect of $\gamma$ on the performance of MR-ADMM, fixing $\eta_i(2k-1) = 1.01^k$: in Fig. \ref{fig:gamma1}\ref{fig:gamma2}, green curves represent the non-private conventional ADMM while other curves represent the private MR-ADMM with different $\gamma$ and each of them illustrates the overall result summarized from 10 independent runs of experiments under the same parameter. The corresponding classification error rates are shown in Fig. \ref{fig:gamma3}. It shows that varying $\gamma$ within a certain range doesn't effect the performance significantly. }}
	\label{fig:gamma}
\end{figure}
\begin{figure}[h]
	\centering   
	\subfigure[Accuracy comparison for different $\eta(2k-1)$ ($\alpha=2$)]{\label{fig:eta1}\includegraphics[trim={2.8cm 0 2.7cm 0.2cm},clip=true, width=0.5\textwidth]{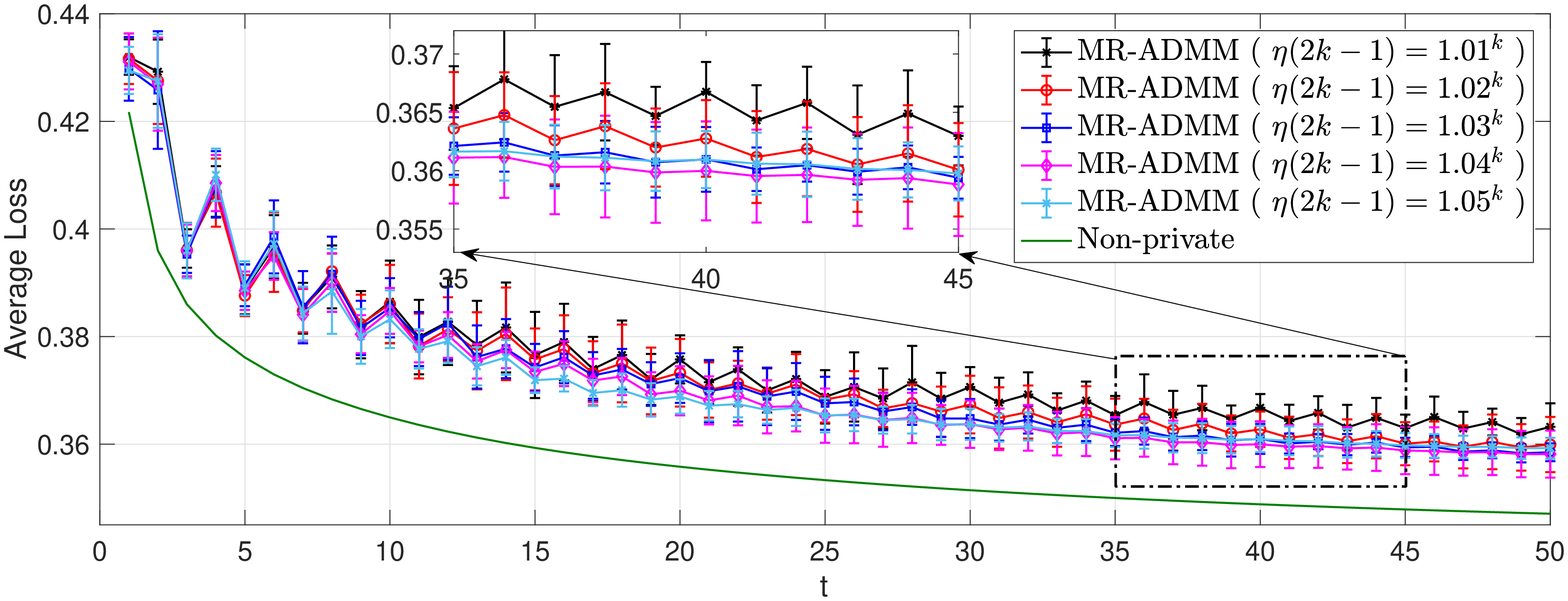}}
	\subfigure[Accuracy comparison for different $\eta(2k-1)$ ($\alpha=1$)]{\label{fig:eta2}\includegraphics[trim={2.8cm 0 2.7cm 0.2cm},clip=true, width=0.5\textwidth]{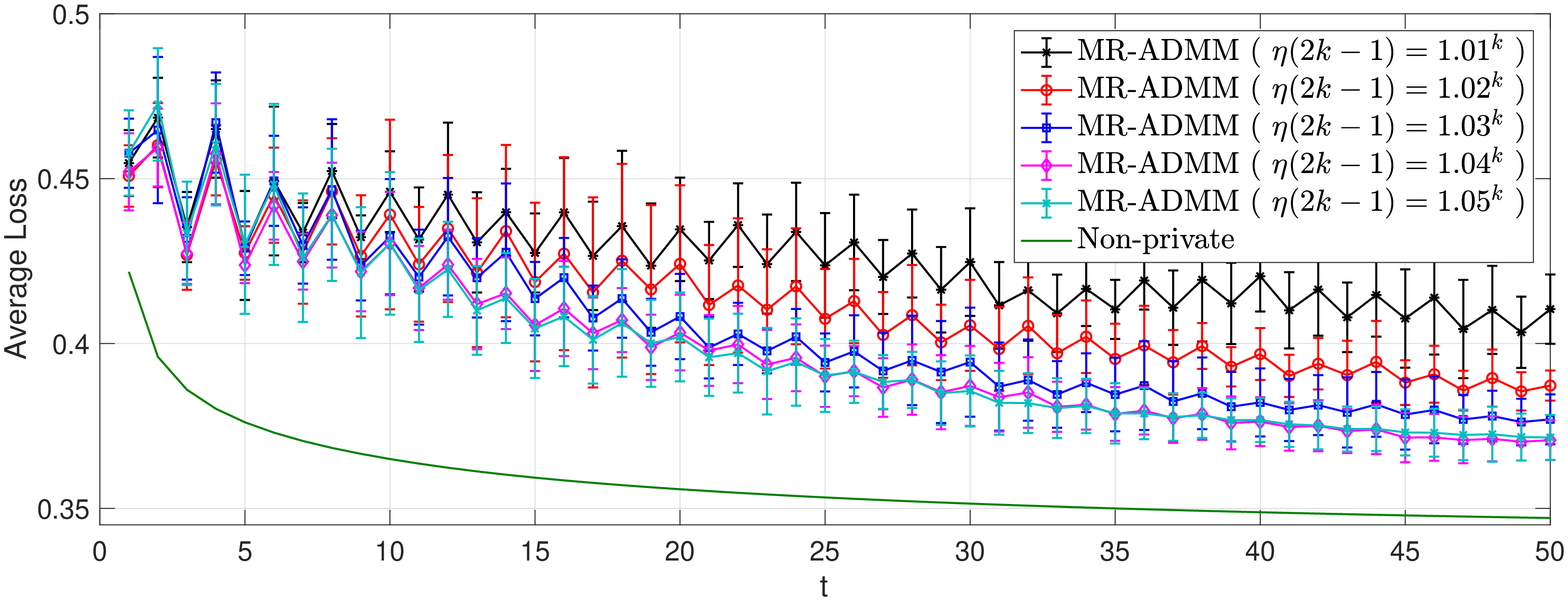}}
	\subfigure[Privacy comparison ($\alpha=2$)]{\label{fig:eta3}\includegraphics[trim={1cm 0 1.8cm 0cm},clip=true, width=0.24\textwidth]{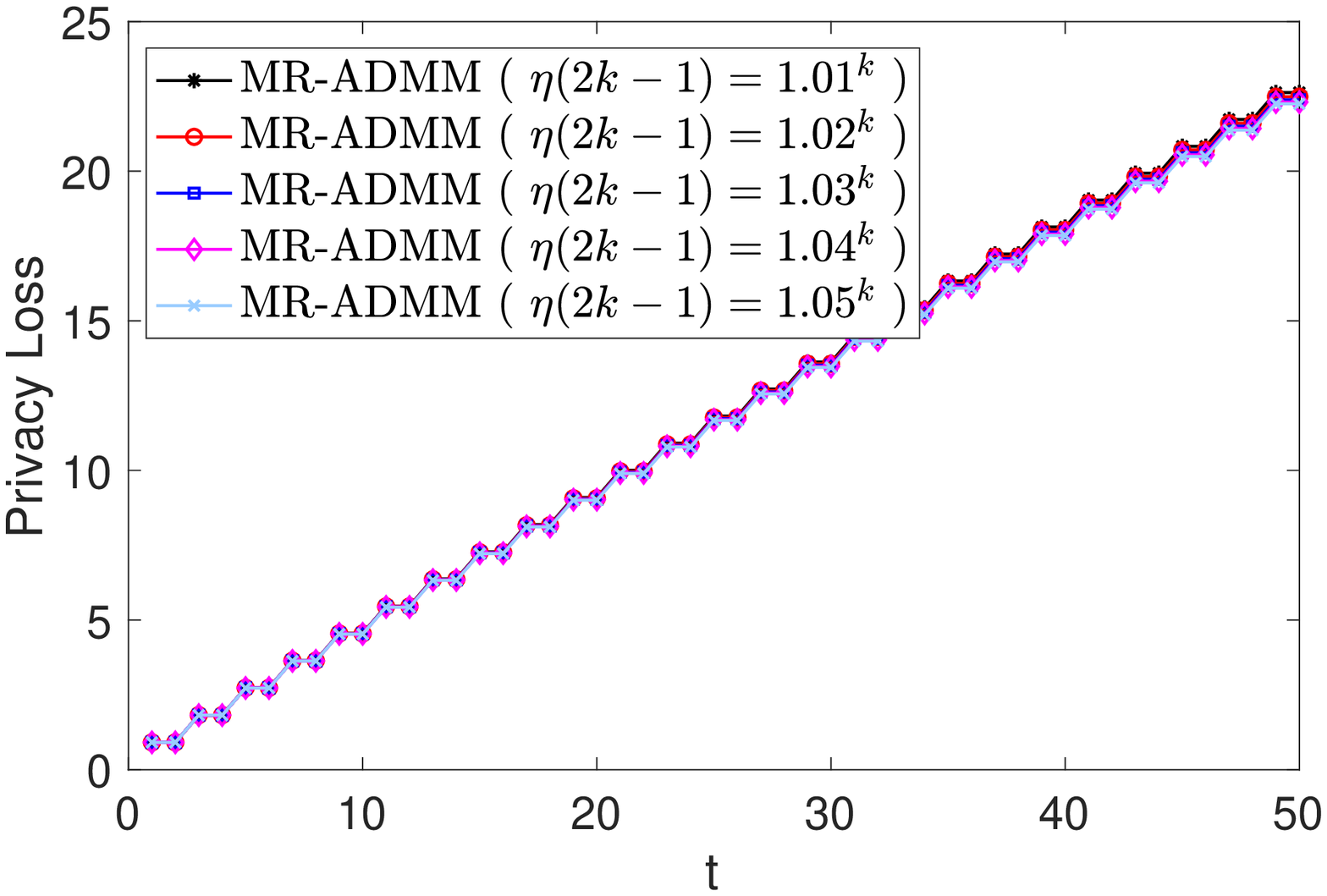}}
	\subfigure[Privacy comparison ($\alpha=1$)]{\label{fig:eta4}\includegraphics[trim={0.7cm 0 1.8cm 0cm},clip=true, width=0.24\textwidth]{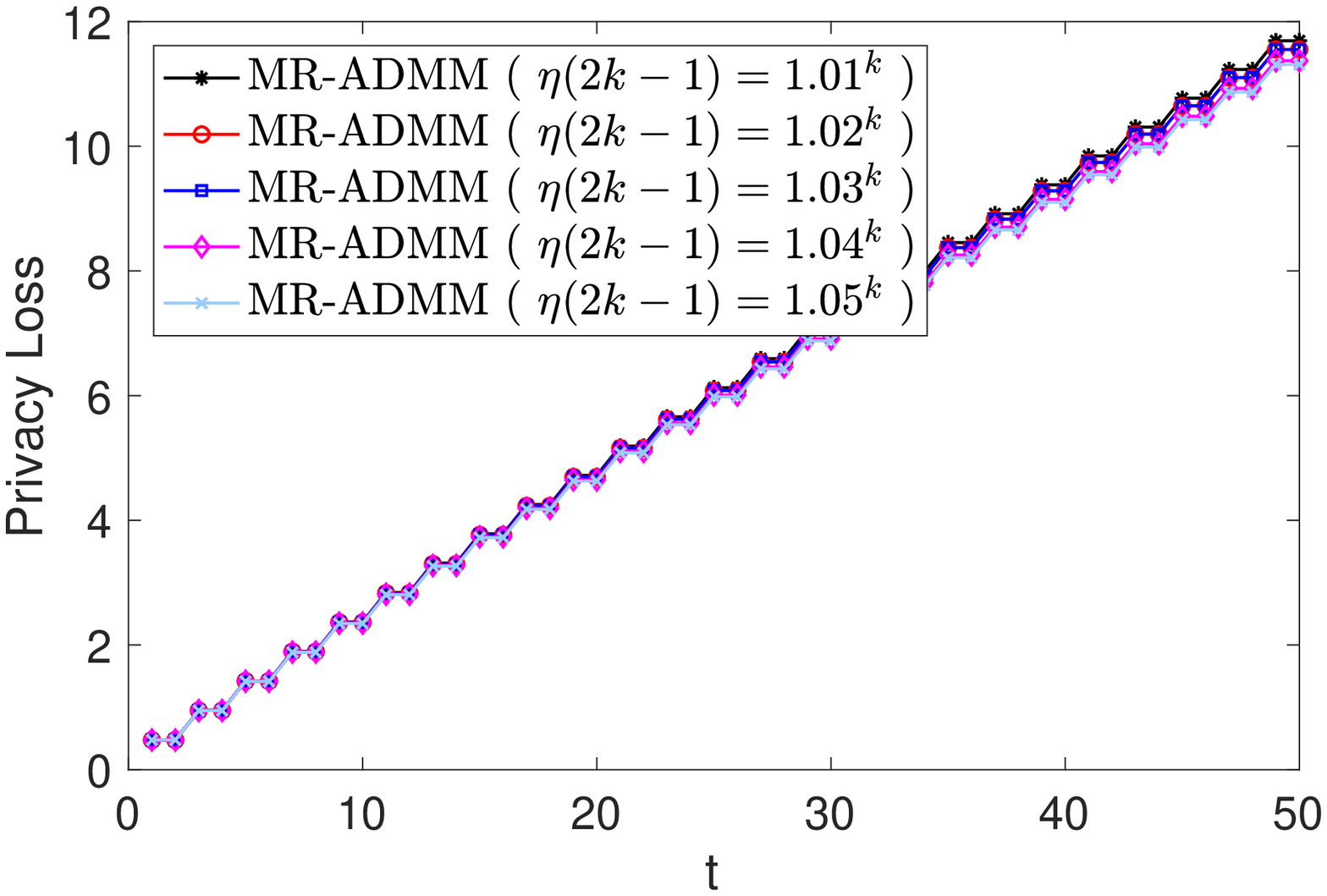}}
	\subfigure[Classification error rate comparison]{\label{fig:eta5}\includegraphics[trim={0cm 0.6cm 0cm 0.9cm},clip=true, width=0.35\textwidth]{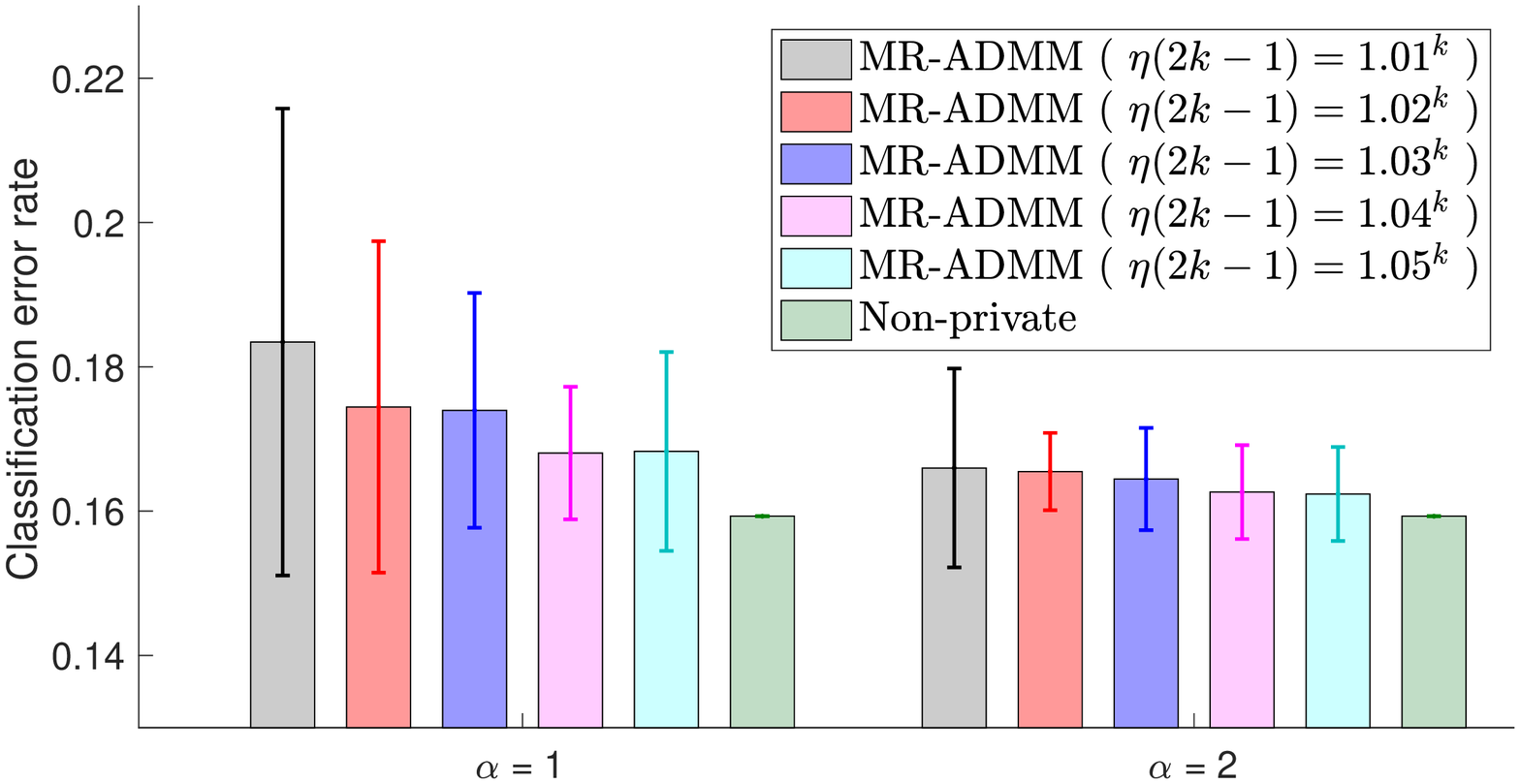}}
	\caption{\rev{The effect of $\eta_i(2k-1)$ on the performance of MR-ADMM, fixing $\gamma=0.5$: in Fig. \ref{fig:eta1}\ref{fig:eta2}, green curves represent the non-private conventional ADMM while other curves represent the private MR-ADMM with different $\eta_i(2k-1) = q_1^k$ ($q_1=1.01,1.02,1.03,1.04,1.05$) and each of them illustrates the overall result summarized from 10 independent runs of experiments under the same parameter. Fig. \ref{fig:eta3}\ref{fig:eta4} illustrate the upper bound of their privacy loss and the corresponding classification error rates are shown in Fig. \ref{fig:eta5}. 
	} }
	\label{fig:eta}
\end{figure}
\begin{figure}
	\centering   
	\subfigure[Accuracy comparison ($\alpha=2$)]{\label{fig:L2}\includegraphics[trim={3cm 0 3cm 0.6cm},clip=true, width=0.48\textwidth]{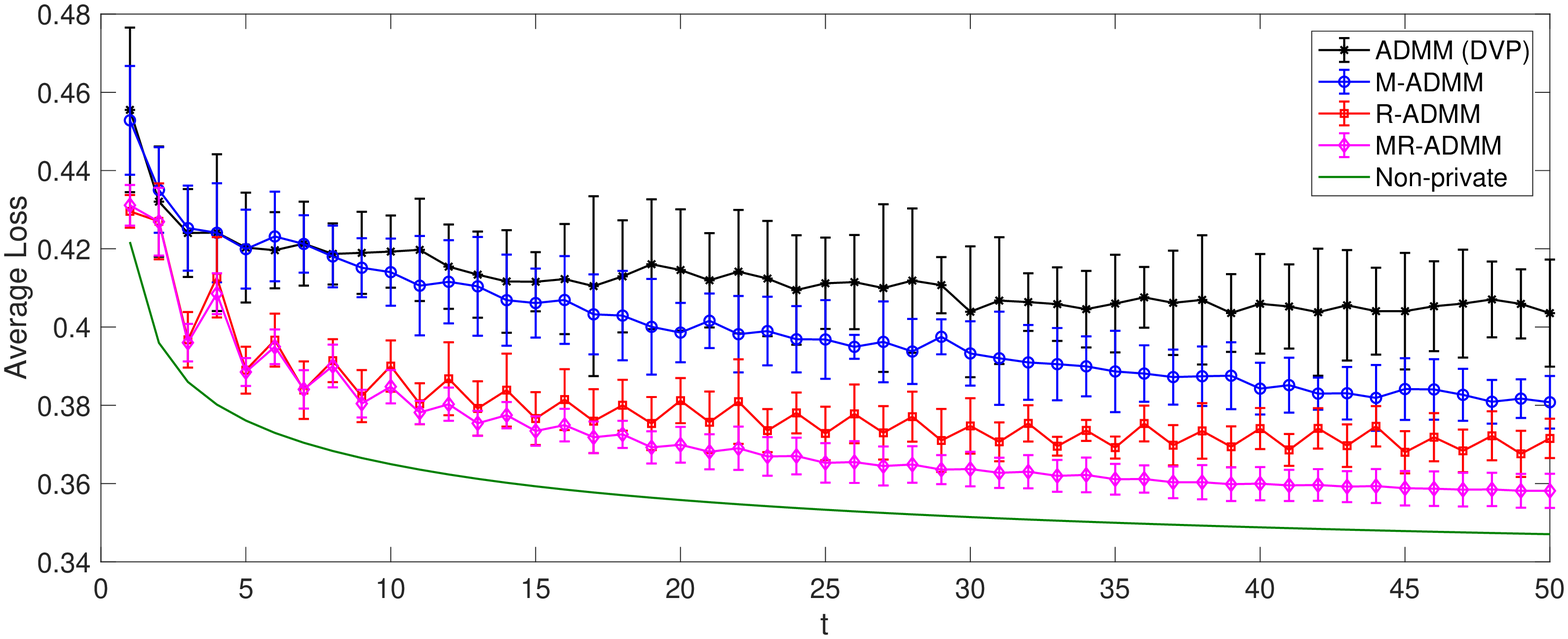}}
	\subfigure[Accuracy comparison ($\alpha=1$)]{\label{fig:L1}\includegraphics[trim={3cm 0 3cm 0.6cm},clip=true, width=0.48\textwidth]{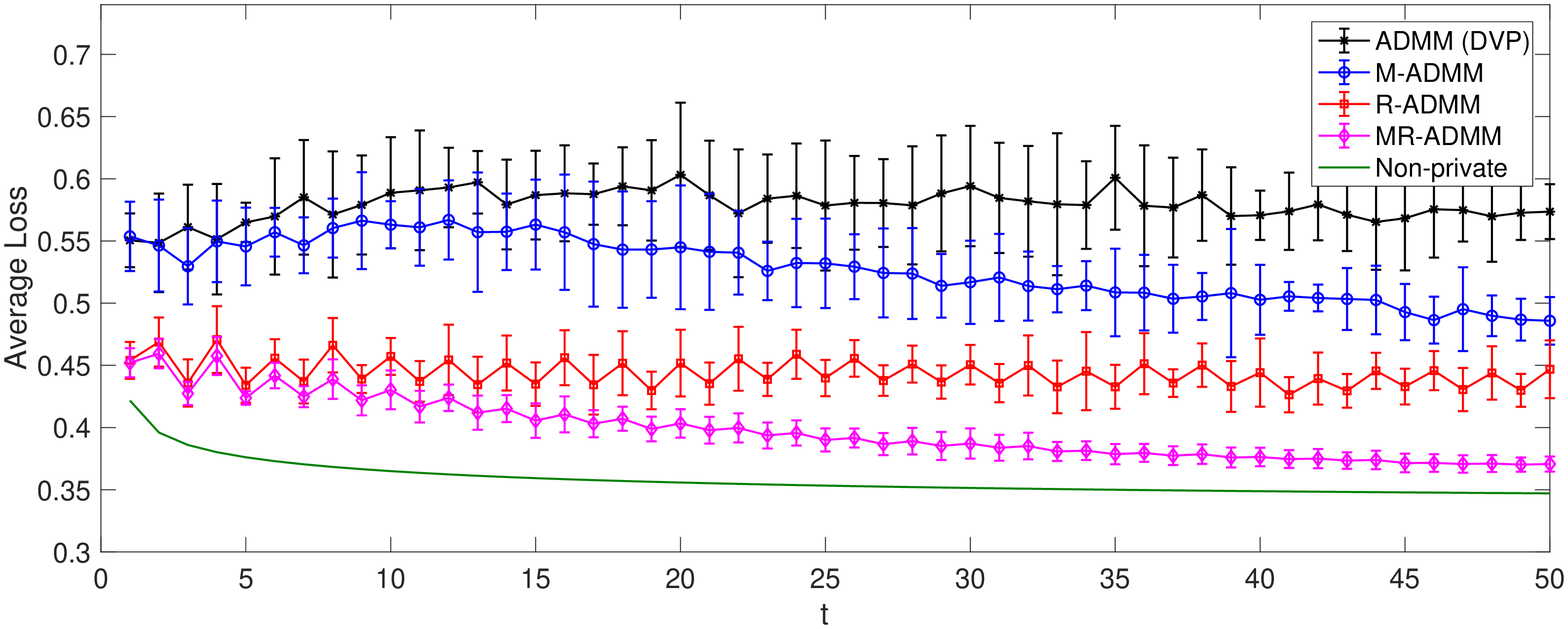}}
	\subfigure[Accuracy comparison ($\alpha=0.5$)]{\label{fig:L05}\includegraphics[trim={3cm 0 3cm 0.6cm},clip=true, width=0.48\textwidth]{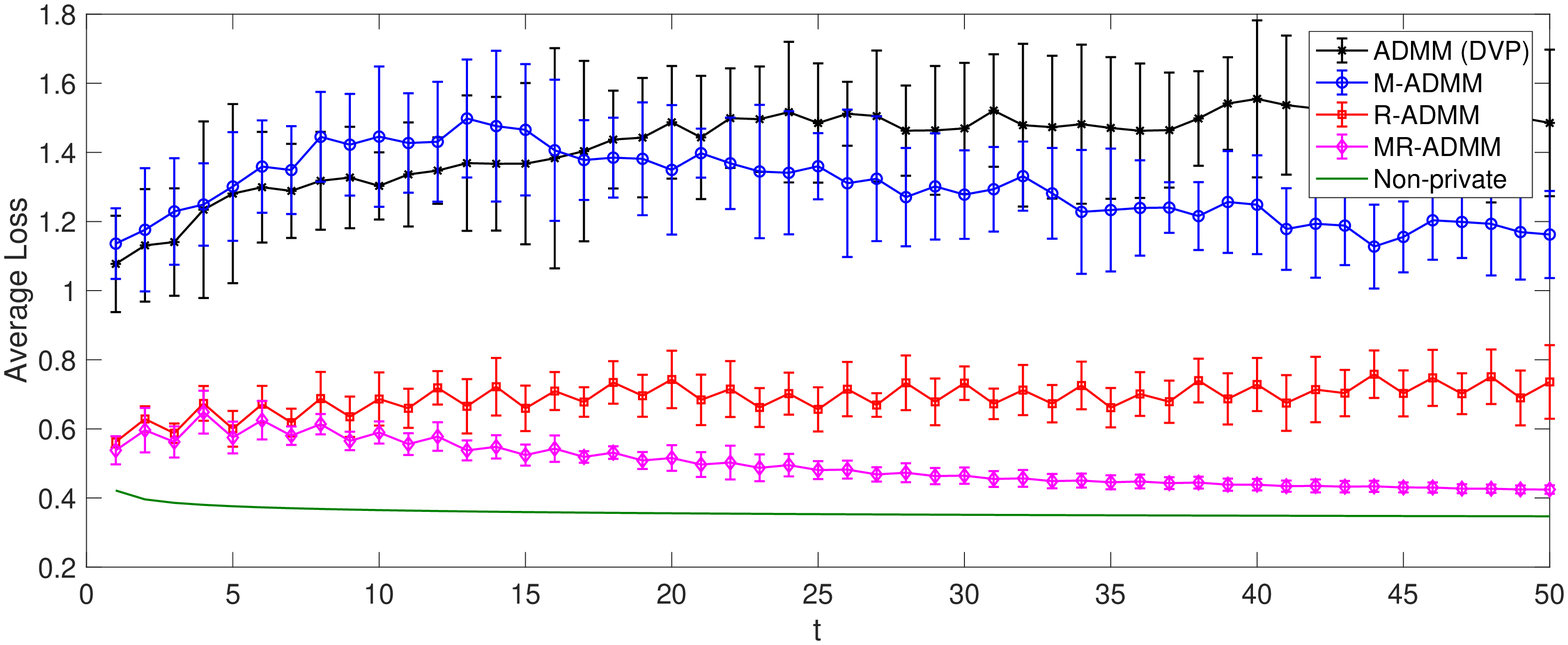}}
	
	\subfigure[ $\alpha=2$]{\label{fig:PL2}\includegraphics[trim={0cm 0 1.1cm 0.7cm},clip=true, width=0.15\textwidth]{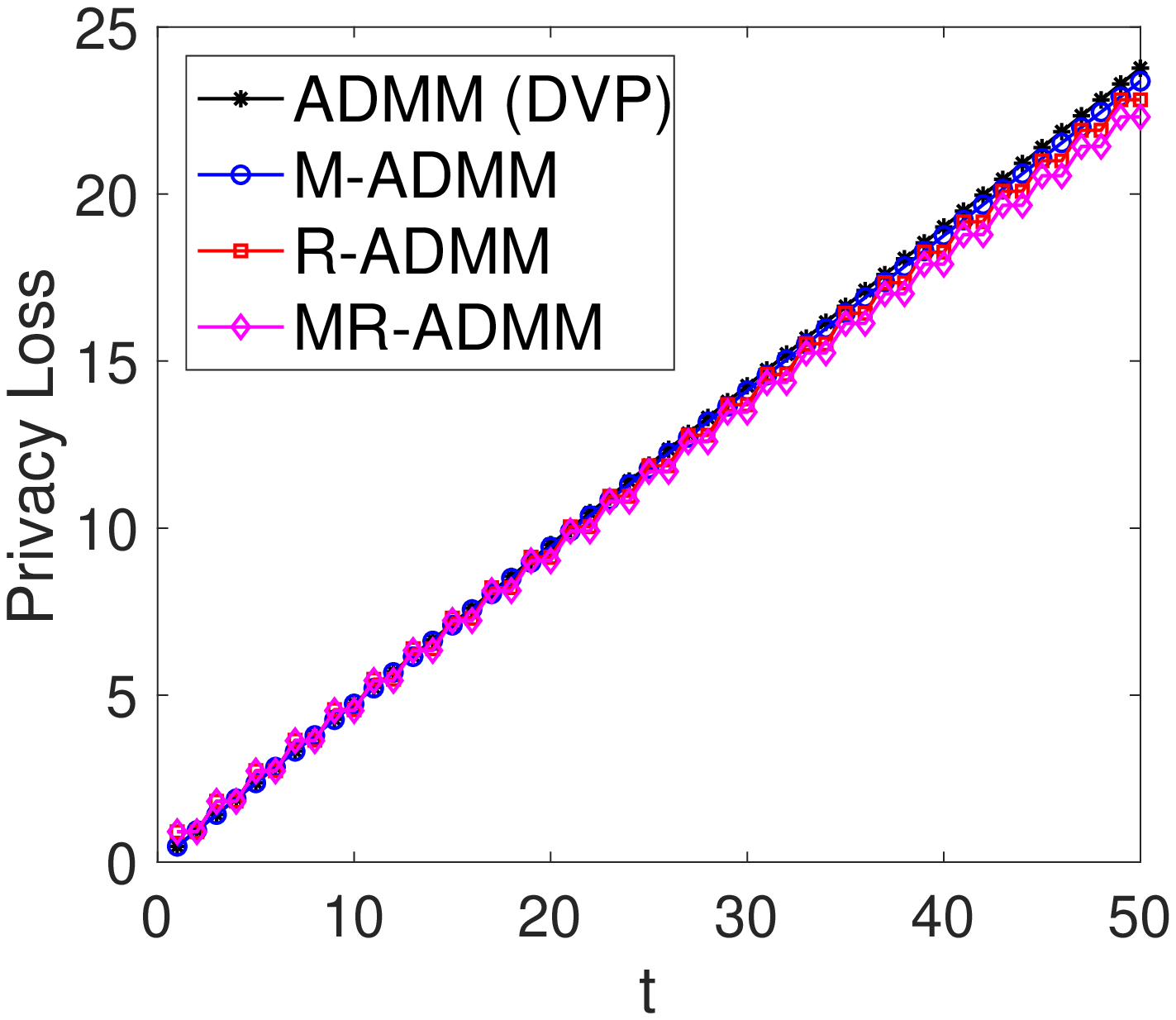}}
	\subfigure[$\alpha=1$]{\label{fig:PL1}\includegraphics[trim={0cm 0 1.1cm 0.7cm},clip=true, width=0.15\textwidth]{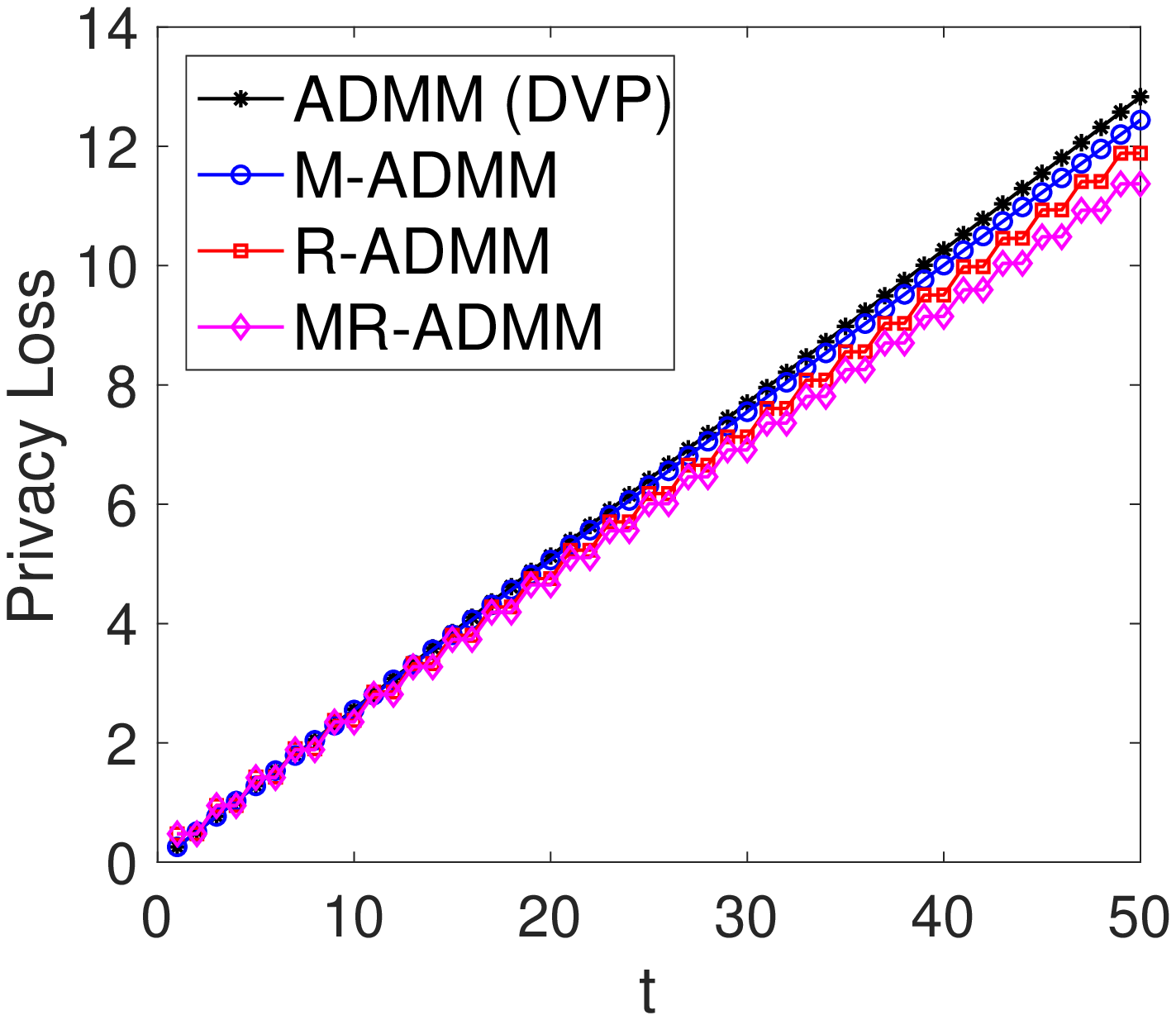}}
	\subfigure[$\alpha=0.5$]{\label{fig:PL05}\includegraphics[trim={0cm 0cm 1.1cm 0.7cm},clip=true, width=0.15\textwidth]{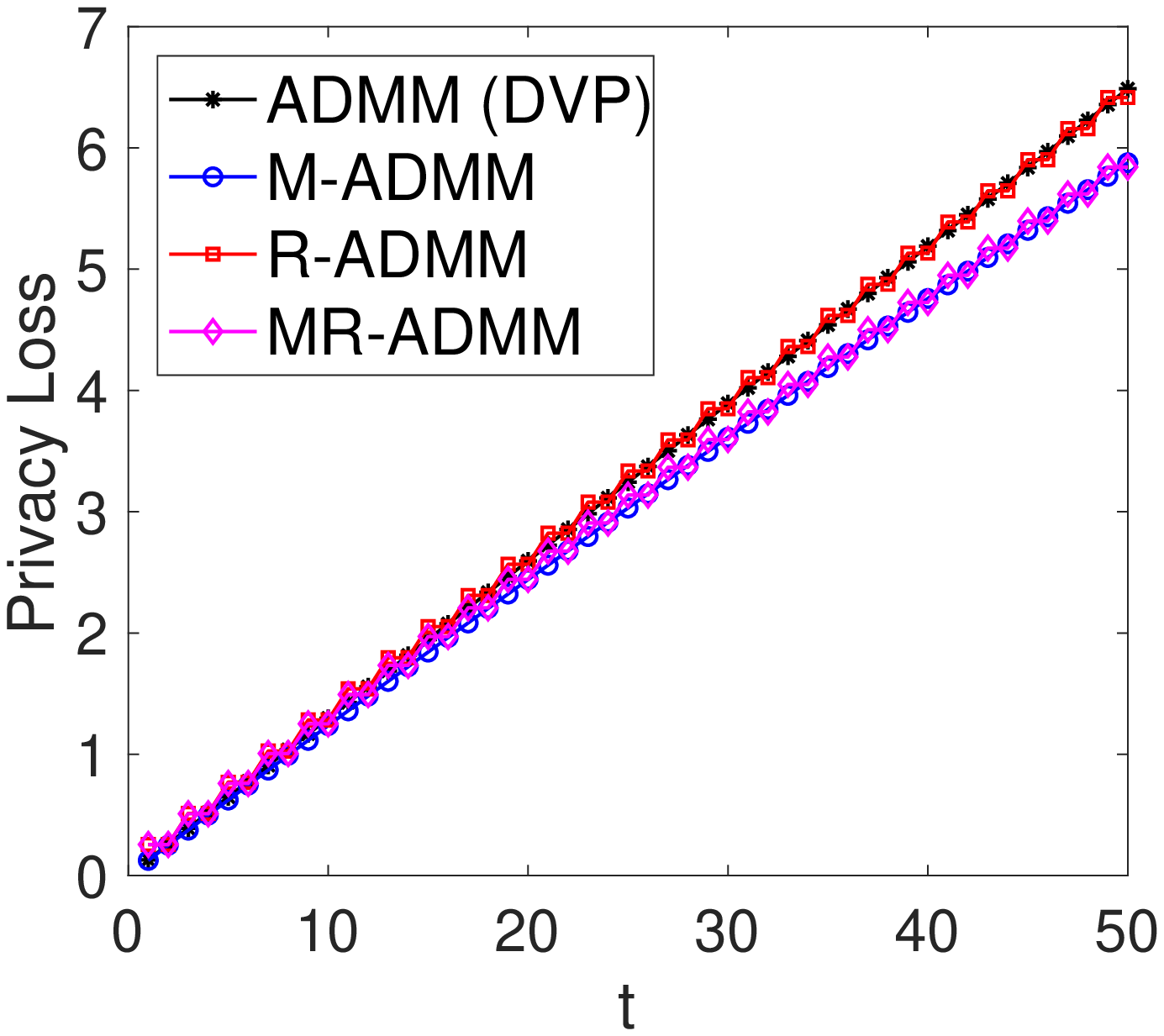}}
	
	\subfigure[Classification error rate comparison]{\label{fig:allRate}\includegraphics[trim={0cm 0.6cm 0cm 1.4cm},clip=true, width=0.36\textwidth]{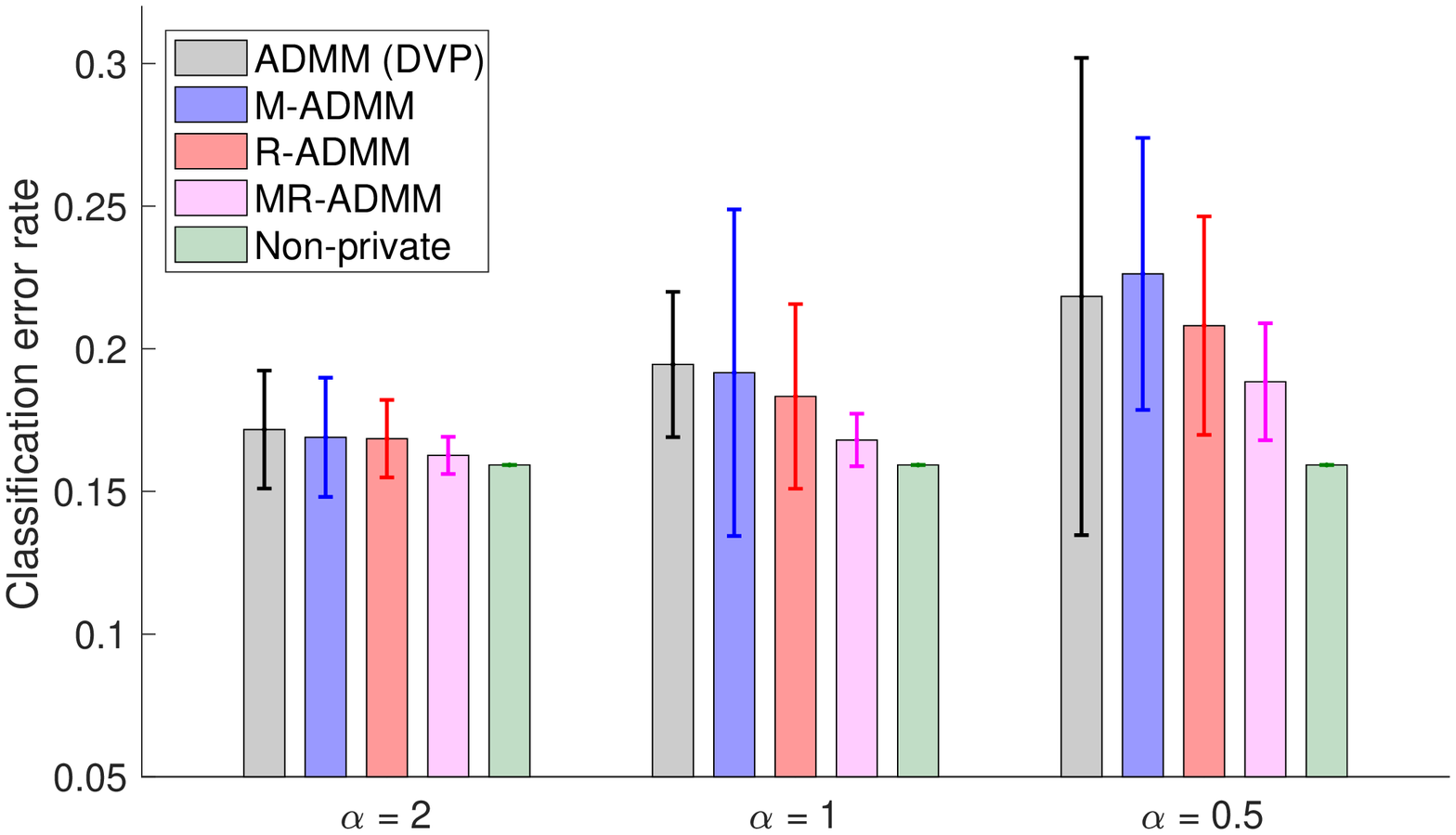}}
	\caption{\rev{Performance comparison: in Fig. \ref{fig:L2}\ref{fig:L1}\ref{fig:L05}, green curves represent the non-private conventional ADMM while other curves represent different private algorithms and each of them illustrates the overall result summarized from 10 independent runs of experiments under the same parameter. M-ADMM (blue) and MR-ADMM (magenta) adopt the varied penalty parameter while ADMM (black) and R-ADMM (red) adopt the fixed $\eta_i(t) = \eta=1$. Fig. \ref{fig:PL2}\ref{fig:PL1}\ref{fig:PL05} illustrate the upper bound of their privacy loss and the corresponding classification error rates are shown in Fig. \ref{fig:allRate}. 
	}}
	\label{fig:compare}
\end{figure}

\section{Numerical Experiments}\label{sec:numerical} 
We use the \textit{Adult} dataset from the UCI Machine Learning Repository \cite{Lichman2013}. It consists of personal information of around 48,842 individuals, including age, sex, race, education, occupation, income, etc. The goal is to predict whether the annual income of an individual is above \$50,000. 

Following the same pre-processing steps as in \cite{xueru}, the final data includes 45,223 individuals, each represented as a 105-dimensional vector of norm at most 1. \rev{We then randomly partition this sample set into a training set (40,000 samples) and a testing set (5,223 samples).  The training samples are then evenly distributed across nodes in a network.}

We use as loss function the logistic loss $\mathscr{L}(z) = \log(1+\exp(-z))$, with $|\mathscr{L}'|\leq 1 $ and $\mathscr{L}'' \leq c_1 = \frac{1}{4}$. 
The regularizer is $R(f_i) = \frac{1}{2}||f_i||_2^2$. 
We measure the accuracy of the algorithm by the average loss over the training set:$$L(t):=\frac{1}{N} \sum_{i=1}^{N}\frac{1}{B_i}\sum_{n=1}^{B_i}\mathscr{L}(y^n_if_i(t)^Tx^n_i), $$ 
\rev{and the classification error rate over the testing set $\mathcal{S}_{test}$:$$E = \frac{\sum_{(x_j,y_j)\in \mathcal{S}_{test}}\textbf{1}(y_j\neq\hat{y}_j)}{\sum_{(x_j,y_j)\in \mathcal{S}_{test}}1}, $$
where $\hat{y}_j$ is the prediction of sample $(x_j,y_j)$ by using the averaged classifier $\bar{f}(t) = \frac{1}{N}\sum_{i=1}^{N}f_i(t)$, and each $f_i(t)$ is the local classifier(primal variable) of node $i$ after $t$ iterations.}\respone{R1.5}\resptwo{R2.5}  

 We measure the privacy of an algorithm by the upper bound: $$P(t):=\underset{i \in \mathscr{N}}{\max}\{\sum_{k=1}^{K}\frac{2C}{B_i}(\frac{1.4c_1}{(\frac{\rho}{N}+2\eta_i(2k-1) V_i)} + \alpha_i(k))\}.$$
The smaller $L(t)$ and $P(t)$, the higher accuracy and stronger privacy guarantee.


\subsection{Convergence of non-private R-ADMM \& MR-ADMM}
Fig. \ref{fig1:a} shows the convergence of R-ADMM with different $\gamma$ and fixed $\eta=0.5$ for a small network ($N=5$) and a large network ($N=20$), both are randomly generated. Due to the linear approximation in even iterations, it's possible to cause an increased average loss as shown in the plot. However, the odd iterations will always compensate this increase; if we only look at the odd iterations, R-ADMM achieves a similar convergence rate as conventional ADMM. $\gamma$ can also be thought of as an extra penalty parameter for each node in even iterations to punish its update, i.e., the difference between $f_i(2k)$ and $f_i(2k-1)$.  Larger $\gamma$ can result in smaller oscillation between even and odd iterations but will also lower the convergence rate.

Fig. \ref{fig1:c}\ref{fig1:b} show the convergence of MR-ADMM with penalty parameters $\eta_i(2k-1)$ increasing at different speed. We see that increasing penalty slows down the convergence, and larger increase in $q_1(i)$ slows it down more. In  \ref{fig1:c}, each node adopts different penalty parameter $\eta_i(2k-1)$ in each iteration while in \ref{fig1:b}, the same penalty parameter is shared among all the nodes. The convergence is attained in both cases. 

\subsection{Private R-ADMM \& MR-ADMM}


\subsubsection{The effect of $\rho$, $\gamma$, $\eta_i(2k-1)$}
We next inspect the accuracy and privacy of the private R-ADMM and MR-ADMM (Algorithm \ref{A2}), and compare it with the private (conventional) ADMM using dual variable perturbation (DVP) \cite{zhang2017}, the private M-ADMM using penalty perturbation (PP) \cite{xueru}. 

\rev{To begin, we first examine the effect of $\rho$ in controlling overfitting. Fig. \ref{fig:rho} shows the classification error rate over the testing set under different $\rho$, where the classifiers are trained with original ADMM and the algorithm runs for 50 iterations. Since the classification error rate is minimized at $\rho\approx0.22$, we will use $\rho=0.22$ \respone{R1.4}\resptwo{R2.5}in the following experiments. } 

{For simplicity of presentation, in the next set of experiments the penalty $\eta_i(t) = \eta(t)$ in both M-ADMM and MR-ADMM and noise $\alpha_i(k) =\alpha, \forall i,k$. We observe similar results when $\alpha_i(t)$, $\eta_i(t)$ vary from node to node.}

For each parameter setting, we perform 10 independent runs of the algorithm, and record both the mean and the range of their accuracy.  Specifically, $L^l(t)$  denotes the average loss over the training dataset in the $t$-th iteration of the $l$-th experiment ($1\leq l \leq 10$). The mean of average loss is given by $L_{mean}(t) = \frac{1}{10}\sum_{l=1}^{10} L^l(t)$ and the range $L_{range}(t) = \underset{1\leq l \leq 10}{\max} L^l(t) - \underset{1\leq l \leq 10}{\min} L^l(t)$. 
The larger the range $L_{range}(t)$ the less stable the algorithm, i.e., under the same parameter setting, the difference in performances (convergence curves) of  two experiments is larger. In the next few plots, $L_{range}(t)$ is shown as the size of a vertical bar centered at $L_{mean}(t)$.  \rev{Similarly, let $E^l$ be the classification error rate over the testing set in the $l$-th experiment, with an average error rate $E_{mean} =  \frac{1}{10}\sum_{l=1}^{10}E^l$ and range $E_{range} = \underset{1\leq l \leq 10}{\max} E^l - \underset{1\leq l \leq 10}{\min} E^l$ shown as the size of a vertical bar centered at $E_{mean}$.} 
Each parameter setting also has a corresponding upper bound on the privacy loss denoted by $P(t)$.

\rev{In the non-private case, $\gamma$ controls the oscillation between even and odd iterations, as well as the convergence rate. We now examine its effect when MR-ADMM is perturbed. Fig. \ref{fig:gamma} shows the average loss over the training set (Fig. \ref{fig:gamma1}\ref{fig:gamma2}) and the classification error rate over the testing set (Fig. \ref{fig:gamma3}) under different $\gamma>0$, noting that the corresponding privacy loss of these cases are the same under the same $\alpha$. It shows that varying $\gamma$ (within a certain range) does not effect performance significantly.  For the next set of experiments, we fix $\gamma = 0.5$.\resptwo{R2.6} }

\rev{The effect of $\eta_i(2k-1)$ on the performance of private MR-ADMM is illustrated in Fig. \ref{fig:eta}, where the pair Fig. \ref{fig:eta1}, \ref{fig:eta3} is for the case when noise parameter is $\alpha = 2$ (low privacy requirement) and the pair Fig. \ref{fig:eta2}, \ref{fig:eta4} is for the case when $\alpha=1$ (high privacy requirement). Although increasing $\eta_i(2k-1)$ over time can decrease the convergence rate of non-private MR-ADMM (Fig. \ref{fig1:c}\ref{fig1:b}), it helps to stabilize the algorithm when MR-ADMM is perturbed and can improve the accuracy while maintain the privacy guarantee. Moreover, the improvement is more significant when algorithm is under higher perturbation (high privacy requirement) and when $\eta_i(2k-1)$ increases faster (within a range). } \resptwo{R2.6}

\subsubsection{Performance comparison among different algorithms}
\rev{Our last set of experiments is conducted to compare the performance of different algorithms with results illustrated in Fig. \ref{fig:compare}. The noise parameters of both MR-ADMM and R-ADMM are set as $\alpha$ shown in the plots, and the noise parameters of conventional ADMM and M-ADMM are chosen respectively such that they have approximately the same total privacy loss bounds. We set $\eta_i(2k-1) = 1.04^k$ in MR-ADMM. We see that both private R-ADMM (red) and private MR-ADMM (magenta) outperform private ADMM (black) and M-ADMM (blue) with higher accuracy and lower privacy loss. In particular, the private MR-ADMM (magenta) has the highest accuracy with the lowest privacy loss among all algorithms; the improvement is more significant with  smaller total privacy loss. This improvement is also illustrated by the classification error rate over the testing set in Fig. \ref{fig:allRate}.    }

\section{Conclusion}\label{sec:conclusion}
In this work, we presented Recycled ADMM (R-ADMM), a modified
version of ADMM that can improve the privacy-utility
tradeoff significantly with less computation. The idea is
to repeatedly use the existing computational results instead
of the original individuals' data to make updates. We also modify R-ADMM (MR-ADMM) by incorporating the idea from \cite{xueru} to further improve the privacy-utility tradeoff of R-ADMM. The idea is to stabilize algorithm by decreasing its step-size, i.e., increasing penalty parameters, over iterations. A sufficient condition for the convergence and the privacy analysis using objective perturbation of two algorithms are established. The experiments on real-world dataset also validate the algorithm.

\bibliographystyle{IEEEtran}
\bibliography{allerton2018_xueru}

\clearpage

\appendices
\begin{figure*}
	\normalsize
	\begin{eqnarray}\label{eq:thmC1}
	\big\langle \hat{f}(t+1)-\hat{f}^*,&-&W(t+1) (D+A) \tilde{D}(t)^{-1}(\nabla \hat{O}(\hat{f}(t),D_{all}) -\nabla \hat{O}(\hat{f}^*,D_{all}))\nonumber \\&+&(I+W(t+1)(D+A) \tilde{D}(t)^{-1})(2\Lambda^*-2\Lambda(t+1))\nonumber \\&+&W(t+1) (D+A) \tilde{D}(t)^{-1}(2\Lambda(t+1)-2\Lambda(t))-W(t+1)(D+A)(\hat{f}(t+1)-\hat{f}(t))\nonumber \\
	&-&W(t+1) (D+A) \tilde{D}^{-1} W(t)(D-A)\hat{f}(t) \big\rangle_F \geq 0 ~.
	\end{eqnarray}
	\hrulefill
	\begin{eqnarray}
&&	\big\langle \hat{f}(t+1)-\hat{f}^*,W(t+1)(D+A) \tilde{D}(t)^{-1}(2\Lambda(t+1)-2\Lambda(t)) -W(t+1) (D+A) \tilde{D}(t)^{-1}W(t)(D-A)\hat{f}(t) \big\rangle_F\nonumber\\
	&=&\big\langle \hat{f}(t+1)-\hat{f}^*,W(t+1)(D+A) \tilde{D}(t)^{-1} W(t)(D-A)(\hat{f}(t+1)-\hat{f}(t)) \big\rangle_F\nonumber\\
	&&+\big\langle \hat{f}(t+1)-\hat{f}^*,W(t+1)(D+A) \tilde{D}(t)^{-1} (W(t+1)-W(t))(D-A)(\hat{f}(t+1)-f^*) \big\rangle_F \nonumber\\&=& \frac{1}{2}||\hat{f}(t+1)-\hat{f}^*||^2_{G_1(t+1)} +  \frac{1}{2}||\hat{f}(t+1)-\hat{f}(t)||^2_{G_1(t+1)}- \frac{1}{2}||\hat{f}(t)-\hat{f}^*||^2_{G_1(t+1)}\nonumber\\
	&&+\big\langle \hat{f}(t+1)-\hat{f}^*,W(t+1)(D+A) \tilde{D}(t)^{-1} (W(t+1)-W(t))(D-A)(\hat{f}(t+1)-f^*) \big\rangle_F~;\label{eq:thmC2}\\
&&	\big\langle \hat{f}(t+1)-\hat{f}^*, (I+W(t+1)(D+A) \tilde{D}(t)^{-1})(2\Lambda^*-2\Lambda(t+1)) \big\rangle_F
	\nonumber \\&=&\big\langle  (W(t+1)(D-A))^{+}(2\Lambda(t+1)-2\Lambda(t)), (I+W(t+1)(D+A) \tilde{D}(t)^{-1})(2\Lambda^*-2\Lambda(t+1)) \big\rangle_F
\nonumber\\&=&\frac{1}{2}||2\Lambda^*-2\Lambda(t)||^2_{G_2(t+1)}-\frac{1}{2}||2\Lambda^*-2\Lambda(t+1)||^2_{G_2(t+1)} - \frac{1}{2}||2\Lambda(t+1)-2\Lambda(t)||^2_{G_2(t+1)}~; \label{eq:thmC3}\\
&&	\langle \hat{f}(t+1)-\hat{f}^*,-W(t+1)(D+A)(\hat{f}(t+1)-\hat{f}(t))\rangle_F\nonumber \\
	&=&\frac{1}{2}||\hat{f}(t)-\hat{f}^*||^2_{W(t+1)(D+A)}-\frac{1}{2}||\hat{f}(t+1)-\hat{f}^*||^2_{W(t+1)(D+A)}-\frac{1}{2}||\hat{f}(t)-\hat{f}(t+1)||^2_{W(t+1)(D+A)}~. \label{eq:thmC4}
	\end{eqnarray}
	\hrulefill
	\begin{eqnarray}\label{eq:thmC5}
&&	\langle \hat{f}(t+1)-\hat{f}^*,-W(t+1)(D+A) \tilde{D}(t)^{-1}(\nabla \hat{O}(\hat{f}(t),D_{all}) -\nabla \hat{O}(\hat{f}^*,D_{all}))\rangle_F
	\nonumber \\&=&\langle \hat{f}(t+1)-\hat{f}(t)+\hat{f}(t) -\hat{f}^*,\nonumber -W(t+1)(D+A) \tilde{D}(t)^{-1}(\nabla \hat{O}(\hat{f}(t),D_{all}) -\nabla \hat{O}(\hat{f}^*,D_{all}))\rangle_F
	\nonumber \\ &\leq& \langle \hat{f}(t)-\hat{f}(t+1),W(t+1) (D+A) \tilde{D}(t)^{-1}(\nabla \hat{O}(\hat{f}(t),D_{all}) -\nabla \hat{O}(\hat{f}^*,D_{all}))\rangle_F \nonumber 
	\nonumber \\&= & \langle W(t+1) (D+A) \sqrt{\tilde{D}(t)^{-1}}(\hat{f}(t)-\hat{f}(t+1)),\sqrt{\tilde{D}(t)^{-1}}(\nabla \hat{O}(\hat{f}(t),D_{all}) -\nabla \hat{O}(\hat{f}^*,D_{all}))\rangle_F~.
	\end{eqnarray}
	\hrulefill
	\begin{eqnarray}\label{eq:thmC8}
	\eqref{eq:thmC5} &\leq& \frac{1}{L}||(\hat{f}(t)-\hat{f}(t+1))||^2_{W(t+1) (D+A) \tilde{D}(t)^{-1}W(t+1) (D+A)} \nonumber \\&+& \frac{L}{4\sigma_{\min}(\tilde{D}(t))} ( \mu||\hat{f}^*-\hat{f}(t+1)||^2_{D_M}+\frac{\mu}{\mu-1}||\hat{f}(t+1)-\hat{f}(t)||^2_{D_M})\nonumber \\
	&=& \frac{1}{2}||(\hat{f}(t)-\hat{f}(t+1))||^2_{\frac{2}{L}W(t+1) (D+A) \tilde{D}(t)^{-1}W(t+1) (D+A)+\frac{L\mu}{2\sigma_{\min}(\tilde{D}(t))(\mu-1)}D_M} \nonumber \\&+& \frac{1}{2}||2\Lambda(t+1)-2\Lambda(t)||^2_{\frac{L\mu}{2\sigma_{\min}(\tilde{D}(t))}((W(t+1)(D-A))^{+})^2D_M} 
	\end{eqnarray}
	\hrulefill
	\begin{eqnarray}\label{eq:thmC9}
	&&\frac{1}{2}||\hat{f}(t)-\hat{f}(t+1)||^2_{W(t+1)(D+A)-G_1(t+1)}
	- \frac{1}{2}||(\hat{f}(t)-\hat{f}(t+1))||^2_{\frac{2}{L}W(t+1) (D+A) \tilde{D}(t)^{-1}W(t+1) (D+A)+\frac{L\mu}{2\sigma_{\min}(\tilde{D}(t))(\mu-1)}D_M} \nonumber \\&+ & \frac{1}{2}||2\Lambda(t+1)-2\Lambda(t)||^2_{G_2(t+1)}- \frac{1}{2}||2\Lambda(t+1)-2\Lambda(t)||^2_{\frac{L\mu}{2\sigma_{\min}(\tilde{D}(t))}((W(t+1)(D-A))^{+})^2D_M}
	\nonumber \\& \leq& \frac{1}{2}||\hat{f}(t+1)-\hat{f}^*||^2_{G_1(t+1)} - \frac{1}{2}||\hat{f}(t)-\hat{f}^*||^2_{G_1(t+1)} +\frac{1}{2}||2\Lambda^*-2\Lambda(t)||^2_{G_2(t+1)}  \nonumber\\&-&\frac{1}{2}||2\Lambda^*-2\Lambda(t+1)||^2_{G_2(t+1)} 
	+\frac{1}{2}||\hat{f}(t)-\hat{f}^*||^2_{W(t+1)(D+A)}-\frac{1}{2}||\hat{f}(t+1)-\hat{f}^*||^2_{W(t+1)(D+A)}\nonumber \\
	&+&\big\langle \hat{f}(t+1)-\hat{f}^*,W(t+1)(D+A) \tilde{D}(t)^{-1} (W(t+1)-W(t))(D-A)(\hat{f}(t+1)-f^*) \big\rangle_F
	\end{eqnarray}	
	\hrulefill
	\begin{eqnarray}\label{eq:thmC18}
&&	\frac{1}{2}||\hat{f}(t)-\hat{f}(t+1)||^2_{R_1(t+1)}+  \frac{1}{2}||2\Lambda(t+1)-2\Lambda(t)||^2_{R_2(t+1)}
 \leq \frac{1}{2}||\hat{f}(t+1)-\hat{f}^*||^2_{G_1(t+1)} - \frac{1}{2}||\hat{f}(t)-\hat{f}^*||^2_{G_1(t+1)} \nonumber \\&+&\frac{1}{2}||2\Lambda^*-2\Lambda(t)||^2_{G_2(t+1)}-\frac{1}{2}||2\Lambda^*-2\Lambda(t+1)||^2_{G_2(t+1)} 
+\frac{1}{2}||\hat{f}(t)-\hat{f}^*||^2_{W(t+1)(D+A)}-\frac{1}{2}||\hat{f}(t+1)-\hat{f}^*||^2_{W(t+1)(D+A)}\nonumber \\&+&
\big\langle \hat{f}(t+1)-\hat{f}^*,W(t+1)(D+A) \tilde{D}(t)^{-1} (W(t+1)-W(t))(D-A)(\hat{f}(t+1)-f^*) \big\rangle_F
	\end{eqnarray}
		\hrulefill
\end{figure*}
\section{Proof of Theorem \ref{thmC1}}\label{App1}
By convexity of $O(f_i,D_i)$, $(f_i^1-{f}^2_i)^T(\nabla O(f_i^1,D_i)-\nabla O({f}^2_i,D_i)) \geq 0$ holds $\forall$ $f_i^1, {f}_i^2$. Let $\langle\cdot,\cdot\rangle_F$ be frobenius inner product of two matrices, there is: $$\langle \hat{f}(t+1)-\hat{f}^*,\nabla \hat{O}(\hat{f}(t+1),D_{all})-\nabla \hat{O}(\hat{f}^*,D_{all})\rangle_F \geq 0$$ According to \eqref{eq:c_2}\eqref{eq:c_4} and \eqref{eq:c_3}, substitute $\nabla\hat{O}(\hat{f}(t+1),D_{all})-\nabla \hat{O}(\hat{f}^*,D_{all})$ and add an extra term $W(t+1) (D+A)\tilde{D}(t)^{-1}(\nabla \hat{O}(\hat{f}^*,D_{all})+2\Lambda^*)=\textbf{0}_{N\times d}$, implies Eqn. \eqref{eq:thmC1}.


To simplify the notation, for a matrix $X$, let $||X||^2_{J} = \langle X, JX \rangle_F$ and $(X)^+$ be the pseudo inverse of $X$.  Define:
\begin{eqnarray}
G_1(t+1) &=&W(t+1) (D+A) \tilde{D}(t)^{-1}W(t)(D-A) ~;\nonumber \\
G_2(t+1) &=&(W(t+1)(D-A))^{+}\nonumber\\&&\cdot(I+W(t+1) (D+A) \tilde{D}(t)^{-1})~.\nonumber 
\end{eqnarray}

 Use \eqref{eq:c_3}\eqref{eq:c_5} and the fact that $\langle A, JB \rangle_F=\langle J^TA, B \rangle_F$, Eqn. \eqref{eq:thmC2}\eqref{eq:thmC3}\eqref{eq:thmC4} hold.
Let $\sqrt{X}$ denote the square root of a symmetric positive semi-definite (PSD) matrix $X$ that is also symmetric PSD. 
Eqn. \eqref{eq:thmC5} holds, 
where the inequality uses the facts that $O(f_i,D_i)$ is convex for all $i$ and that the matrix $W(t+1)(D+A) \tilde{D}(t)^{-1}$ is positive definite.

According to  \eqref{eq:assume1} in Assumption 3, 
 define the matrix $D_M = \textbf{diag}([M_1^2;M_2^2;\cdots;M_N^2])\in \mathbb{R}^{N \times N}$, it implies
$||\nabla \hat{O}(\hat{f}^1,D_{all}) -\nabla \hat{O}(\hat{f}^2,D_{all})||^2_F \leq \langle \hat{f}^1 - \hat{f}^2,D_M(\hat{f}^1 - \hat{f}^2) \rangle_F$. 
Since $\langle A, B \rangle_F\leq \frac{1}{L}||A||^2_F + \frac{L}{4}||B||_F^2$ holds for any $L>0$, there is:
\begin{eqnarray}\label{eq:thmC7}
&&\eqref{eq:thmC5} \nonumber \\&\leq& \frac{1}{L}||W(t+1)(D+A) \sqrt{\tilde{D}(t)^{-1}}(\hat{f}(t)-\hat{f}(t+1))||^2_F\nonumber \\ &+& \frac{L}{4}||\sqrt{\tilde{D}(t)^{-1}}(\nabla \hat{O}(\hat{f}(t),D_{all}) -\nabla \hat{O}(\hat{f}^*,D_{all}))||_F^2\nonumber \\
&\leq& \frac{1}{L}||(\hat{f}(t)-\hat{f}(t+1))||^2_{W(t+1) (D+A) \tilde{D}(t)^{-1}W(t+1) (D+A)} \nonumber \\&+& \frac{L\sigma_{\max}(\tilde{D}(t)^{-1})}{4} ||\nabla \hat{O}(\hat{f}(t),D_{all}) -\nabla \hat{O}(\hat{f}^*,D_{all})||_F^2\nonumber \\
&= &\frac{1}{L}||(\hat{f}(t)-\hat{f}(t+1))||^2_{W(t+1) (D+A) \tilde{D}(t)^{-1}W(t+1) (D+A)}\nonumber \\ &+& \frac{L}{4\sigma_{\min}(\tilde{D}(t))} ||\hat{f}^*-\hat{f}(t)||^2_{D_M}
\end{eqnarray}
where $\sigma_{\max}(\cdot)$, $\sigma_{\min}(\cdot)$ denote the largest and smallest singular value of a matrix respectively. Since for any $\mu > 1$ and any matrices $C_1$, $C_2$, $J$ with the same dimensions, there is $||C_1+C_2||^2_J \leq \mu||C_1||^2_J+ \frac{\mu}{\mu - 1}||C_2||^2_J$. which implies:
\begin{eqnarray}
||\hat{f}^*-\hat{f}(t)||^2_{D_M} = ||\hat{f}^*-\hat{f}(t+1)+\hat{f}(t+1)-\hat{f}(t)||^2_{D_M}
\nonumber \\\leq \mu||\hat{f}^*-\hat{f}(t+1)||^2_{D_M}+\frac{\mu}{\mu-1}||\hat{f}(t+1)-\hat{f}(t)||^2_{D_M}\nonumber
\end{eqnarray}
Plug into \eqref{eq:thmC7} and use \eqref{eq:c_3}\eqref{eq:c_5} gives Eqn. \eqref{eq:thmC8}.

Combine \eqref{eq:thmC2}\eqref{eq:thmC3}\eqref{eq:thmC4}\eqref{eq:thmC8}, \eqref{eq:thmC1} becomes Eqn. \eqref{eq:thmC9}.
Suppose the following two conditions hold for all $t$ under some constants $L>0$ and $\mu>1$:
\begin{eqnarray*}
&(i)&I+W(t+1)(D+A) \tilde{D}(t)^{-1}\nonumber\\&& \succ \frac{L\mu}{2\sigma_{\min}(\tilde{D}(t))}(W(t+1)(D-A))^{+}D_M ~;\label{eq:thmC14}\\
&(ii)&W(t+1)(D+A)\nonumber\\&&\succ W(t+1) (D+A) \tilde{D}(t)^{-1} \Big(W(t)(D-A)\nonumber \\&&+\frac{2}{L}W(t+1)(D+A)\Big) +\frac{L\mu}{2\sigma_{\min}(\tilde{D}(t))(\mu-1)}D_M ~.\label{eq:thmC15}
\end{eqnarray*}

Substitute $G_1(t+1)$ and $G_2(t+1)$, define $R_1(t+1)$ and $R_2(t+1)$ as \eqref{eq:thmC16}\eqref{eq:thmC17}. By conditions \textit{(i)(ii)}, both $R_1(t+1)$ and $R_2(t+1)$ are positive definite.
\begin{eqnarray}
R_1(t+1) &=& W(t+1)(D+A)-G_1(t+1)\nonumber \\&-&\frac{2}{L}W(t+1) (D+A) \tilde{D}(t)^{-1}W(t+1) (D+A)\nonumber \\&-&\frac{L\mu}{2\sigma_{\min}(\tilde{D}(t))(\mu-1)}D_M\succ \textbf{0}_{N\times N} ~;\label{eq:thmC16}\\
R_2(t+1) &=&- \frac{L\mu}{2\sigma_{\min}(\tilde{D}(t))}((W(t+1)(D-A))^{+})^2D_M  \nonumber \\&+&G_2(t+1)\succ \textbf{0}_{N\times N}~.\label{eq:thmC17}
\end{eqnarray}
Eqn. \eqref{eq:thmC9} becomes Eqn. \eqref{eq:thmC18}.

Since $W(t+1)$, $W(t)$ and $\tilde{D}(t)$ are all diagonal matrices of the same size, define new diagonal matrix $D^{new}_1(t+1)$ with $D^{new}_1(t+1)_{ii} = \frac{\eta_i(t+1)\eta_i(t)}{2\eta_i(t)V_i+\gamma}$, then $G_1(t+1)$ can be rewritten as:
$$G_1(t+1) = D^{new}_1(t+1)(D+A)(D-A).$$

Consider
\begin{eqnarray}
&&\frac{1}{2}||\hat{f}(t+1)-\hat{f}^*||^2_{G_1(t+1)} - \frac{1}{2}||\hat{f}(t)-\hat{f}^*||^2_{G_1(t+1)} \nonumber \\
&=& \frac{1}{2}||\hat{f}(t+1)-\hat{f}^*||^2_{G_1(t+1)} - \frac{1}{2}||\hat{f}(t)-\hat{f}^*||^2_{G_1(t)} \nonumber \\
&+&  \frac{1}{2}||\hat{f}(t)-\hat{f}^*||^2_{G_1(t)} 
-  \frac{1}{2}||\hat{f}(t)-\hat{f}^*||^2_{G_1(t+1)}  \nonumber 
\end{eqnarray}

If $\eta_i(t+1)\geq \eta_i(t)$, $\forall t,i$, then $D^{new}_1(t+1)_{ii} \geq D^{new}_1(t)_{ii} $. Therefore, $G_1(t+1)-G_1(t) \succeq 0$. Let $U_1 = \underset{i,t,k}{\text{sup}}|(f_i(t)-f_c^*)_k| \in \mathbb{R}$ be the finite upper bound over all components $k$, all nodes $i$ and all iterations $t$, then 
\begin{eqnarray}
&&\frac{1}{2}||\hat{f}(t)-\hat{f}^*||^2_{G_1(t)} 
-  \frac{1}{2}||\hat{f}(t)-\hat{f}^*||^2_{G_1(t+1)}  \nonumber\\
&=& \frac{1}{2}\text{Tr}((\hat{f}(t)-\hat{f}^*)^T(G_1(t)-G_1(t+1))(\hat{f}(t)-\hat{f}^*)) \nonumber\\
&\leq&  \frac{1}{2}U_1^2(||\textbf{1}_{N\times d}||^2_{G_1(t+1)}-||\textbf{1}_{N\times d}||^2_{G_1(t)} )\nonumber 
\end{eqnarray}
where $\textbf{1}_{N\times d}$ is the matrix of size $N$ by $d$ with 1 on all the entries.

Therefore, 
\begin{eqnarray}
&&\frac{1}{2}||\hat{f}(t+1)-\hat{f}^*||^2_{G_1(t+1)} - \frac{1}{2}||\hat{f}(t)-\hat{f}^*||^2_{G_1(t+1)} \nonumber \\
&\leq &\frac{1}{2}||\hat{f}(t+1)-\hat{f}^*||^2_{G_1(t+1)} - \frac{1}{2}||\hat{f}(t)-\hat{f}^*||^2_{G_1(t)} \nonumber \\
&+& \frac{1}{2}U_1^2(||\textbf{1}_{N\times d}||^2_{G_1(t+1)}-||\textbf{1}_{N\times d}||^2_{G_1(t)} ) \nonumber
\end{eqnarray}

Similarly, $(W(t+1)-W(t))(D+A) \succeq 0 $ holds if $\eta_i(t+1)\geq \eta_i(t)$, $\forall t,i$, and the following holds.
\begin{eqnarray*}
\frac{1}{2}||\hat{f}(t)-\hat{f}^*||^2_{W(t+1)(D+A)}-\frac{1}{2}||\hat{f}(t+1)-\hat{f}^*||^2_{W(t+1)(D+A)} \nonumber\\
\leq \frac{1}{2}||\hat{f}(t)-\hat{f}^*||^2_{W(t)(D+A)}-\frac{1}{2}||\hat{f}(t+1)-\hat{f}^*||^2_{W(t+1)(D+A)}\nonumber\\
+ \frac{1}{2} U_1^2(||\textbf{1}_{N\times d}||^2_{W(t+1)(D+A)}-||\textbf{1}_{N\times d}||^2_{W(t)(D+A)} )
\end{eqnarray*}

Similarly, if $\eta_i(t+1)\geq \eta_i(t)$, $\forall t,i$, $G_2(t)-G_2(t+1)\succeq 0$. Let $U_2 = \underset{i,t,k}{\text{sup}}|(\lambda_i(t)-\lambda_i^*)_k| \in \mathbb{R}$ be the finite upper bound over all components $k$, all nodes $i$ and all iterations $t$, there is:
\begin{eqnarray}
&&\frac{1}{2}||2\Lambda^*-2\Lambda(t)||^2_{G_2(t+1)}-\frac{1}{2}||2\Lambda^*-2\Lambda(t+1)||^2_{G_2(t+1)} \nonumber \\
&\leq& \frac{1}{2}||2\Lambda^*-2\Lambda(t)||^2_{G_2(t)}-\frac{1}{2}||2\Lambda^*-2\Lambda(t+1)||^2_{G_2(t+1)} \nonumber \\
&+&  \frac{1}{2}U_2^2(||\textbf{1}_{N\times d}||^2_{G_2(t)}-||\textbf{1}_{N\times d}||^2_{G_2(t+1)} ) \nonumber
\end{eqnarray}

If $\eta_i(t+1)\geq \eta_i(t)$, $\forall t,i$, let $\overline{\sigma}_{\max} = \underset{t}{\max} \sigma_{\max}(W(t+1)(D+A) \tilde{D}(t)^{-1} (D-A))$, then there is:
\begin{eqnarray}
\big\langle \hat{f}(t+1)-\hat{f}^*,W(t+1)(D+A) \tilde{D}(t)^{-1}\nonumber\\ \cdot(W(t+1)-W(t))(D-A)(\hat{f}(t+1)-f^*) \big\rangle_F\nonumber\\
\leq  \overline{\sigma}_{\max}U_1^2(||\textbf{1}_{N\times d}||^2_{W(t+1)}-||\textbf{1}_{N\times d}||^2_{W(t)} ) \nonumber
\end{eqnarray}

Sum up \eqref{eq:thmC18} over $t$ from $0$ to $+\infty$ leads to:
\begin{eqnarray}\label{eq:thmC19}
&&\sum_{t=0}^{\infty}\{||\hat{f}(t)-\hat{f}(t+1)||^2_{R_1(t+1)}\nonumber\\&&+  ||2\Lambda(t+1)-2\Lambda(t)||^2_{R_2(t+1)}\} \nonumber \\ &\leq& ||\hat{f}(0)-\hat{f}^*||^2_{W(0)(D+A)}-||\hat{f}(+\infty)-\hat{f}^*||^2_{W(+\infty)(D+A)}\nonumber \\ &+&||\hat{f}(+\infty)-\hat{f}^*||^2_{G_1(+\infty)} - ||\hat{f}(0)-\hat{f}^*||^2_{G_1(0)} \nonumber \\&+&||2\Lambda^*-2\Lambda(0)||^2_{G_2(0)}-||2\Lambda^*-2\Lambda(+\infty)||^2_{G_2(+\infty)} \nonumber \\
&+&U_1^2(||\textbf{1}_{N\times d}||^2_{G_1(+\infty)}-||\textbf{1}_{N\times d}||^2_{G_1(0)} )  \nonumber \\
&+&U_1^2(||\textbf{1}_{N\times d}||^2_{W(+\infty)(D+A)}-||\textbf{1}_{N\times d}||^2_{W(0)(D+A)} )  \nonumber \\
&+&U_2^2(||\textbf{1}_{N\times d}||^2_{G_2(0)}-||\textbf{1}_{N\times d}||^2_{G_2(+\infty)} )  \nonumber \\
&+&2\overline{\sigma}_{\max}U_1^2(||\textbf{1}_{N\times d}||^2_{W(+\infty)}-||\textbf{1}_{N\times d}||^2_{W(0)} ) 
\end{eqnarray}
The RHS of \eqref{eq:thmC19} is finite, implies that $\lim_{t\rightarrow\infty}\{||\hat{f}(t)-\hat{f}(t+1)||^2_{R_1(t+1)}+  ||2\Lambda(t+1)-2\Lambda(t)||^2_{R_2(t+1)}\} = 0$. Since $R_1(t+1)$, $R_2(t+1)$ are not unique, by \eqref{eq:thmC16}\eqref{eq:thmC17}, it requires $\lim_{t\rightarrow\infty}||\hat{f}(t)-\hat{f}(t+1)||^2_{R_1(t+1)}=0$ and $\lim_{t\rightarrow\infty}||2\Lambda(t+1)-2\Lambda(t)||^2_{R_2(t+1)} = 0$ should hold for all possible $R_1(t+1)$, $R_2(t+1)$. Therefore, $\lim_{t\rightarrow\infty}(\hat{f}(t)-\hat{f}(t+1))=\textbf{0}_{N\times d}$ and $\lim_{t\rightarrow\infty}(2\Lambda(t+1)-2\Lambda(t)) = \textbf{0}_{N\times d}$ should hold. $(\hat{f}(t),\Lambda(t))$ converges to the stationary point $(\hat{f}^s,\Lambda^s)$. Now show that the stationary point $(\hat{f}^s,\Lambda^s)$ is the optimal point $(\hat{f}^*,\Lambda^*)$.

Take the limit of both sides of \eqref{eq:c_2}\eqref{eq:c_3} yield: 
\begin{eqnarray}
(I+W(t+1) (D+A) \tilde{D}(t)^{-1})\nonumber\\\cdot(\nabla \hat{O}(\hat{f}^s,D_{all}) +2\Lambda^s)=\textbf{0}_{N\times d}~;\label{eq:thmC24}\\
(D-A)\hat{f}^s = \textbf{0}_{N\times d} \label{eq:thmC25}~.
\end{eqnarray}
Since $I+W(t+1) (D+A) \tilde{D}(t)^{-1}\succ \textbf{0}_{N\times N}$, to satisfy \eqref{eq:thmC24}, $\nabla \hat{O}(\hat{f}^s,D_{all}) +2\Lambda^s=\textbf{0}_{N\times d}$ must hold.

Compare with \eqref{eq:c_4}\eqref{eq:c_5} in Lemma \ref{lemmaP1} and observe that $(\hat{f}^s,\Lambda^s)$ satisfies the optimality condition and is thus the optimal point. Therefore, $(\hat{f}(t),\Lambda(t))$ converges to $(\hat{f}^*,\Lambda^*)$.
\section{Proof of Lemma \ref{lemmaP1}}\label{App_2}
Consider the private MR-ADMM up to $2k$-th iteration. In $(2k-1)$-th iteration, the primal variable is updated via \eqref{eq:P_modify_2}, By KKT condition:
\begin{eqnarray}\label{eq:lemma1}
 \nabla O(f_i(2k-1),D_i) + \epsilon_i(2k-1)=-2\lambda_i(2k-2)\nonumber \\  - \eta_i(2k-1) \sum_{j \in \mathscr{V}_i}(2f_i(2k-1)-f_i(2k-2)-f_j(2k-2))
\end{eqnarray}
Given $\{f_i(t)\}_{i=1}^N$ for $t\leq 2k-2$, $\{\lambda_i(2k-2)\}_{i=1}^N$ are also given. RHS of \eqref{eq:lemma1} can be calculated completely after releasing $\{f_i(k-1)\}_{i=1}^N$, i.e., the information of $\nabla O(f_i(2k-1),D_i) + \epsilon_i(2k-1)$ is completely released during $(2k-1)$-th iteration. Suppose the private MR-ADMM satisfies $\beta_{2k-1}$-differential privacy during $(2k-1)$ iterations, then in $(2k)$-th iterations, by \eqref{eq:P_modify_3}:
\begin{eqnarray}\label{eq:lemma2}
f_i(2k)=f_i(2k-1) - \frac{1}{2\eta V_i+\gamma}\{\nabla O(f_i(2k-1),D_i) \nonumber \\+\epsilon_i(2k-1) +2\lambda_i(2k-1)\nonumber\\+\eta_i(2k-1)\sum_{j\in \mathscr{V}_i}(f_i(2k-1)-f_j(2k-1))\}\nonumber
\end{eqnarray}
which is a deterministic mapping taking the outputs from $(2k-1)$-th iteration as input. Because the differential privacy is immune to post-processing \cite{dwork2014algorithmic}, releasing $\{f_i(2k)\}_{i=1}^N$ doesn't increase the privacy loss, i.e., the total privacy loss up to $(2k)$-th iteration can still be bounded by $\beta_{2k-1}$.  

\section{Proof of Theorem \ref{thmP}}\label{App_3}
Use the uppercase letters $X$ and lowercase letters $x$ to denote random variables and the corresponding realizations, and use $\mathscr{F}_{X}(\cdot)$ to denote its probability distribution.


For two neighboring datasets $D_{all}$ and $\hat{D}_{all}$ of the network, by Lemma \ref{lemmaP1}, the total privacy loss is only contributed by odd iterations. Thus, the ratio of joint probabilities (privacy loss) is given by:
\begin{eqnarray}\label{thmP1}
\frac{\mathscr{F}_{F(0:2K)}(\{f(r)\}_{r=0}^2K|D_{all})}{\mathscr{F}_{F(0:2K)}(\{f(r)\}_{r=0}^2K|\hat{D}_{all})} = \frac{\mathscr{F}_{F(0)}(f(0)|D_{all})}{\mathscr{F}_{F(0)}(f(0)|\hat{D}_{all})}  \nonumber\\ \cdot \prod^K_{k=1}\frac{\mathscr{F}_{F(2k-1)}(f(2k-1)|\{f(r)\}_{r=0}^{2k-2},D_{all})}{\mathscr{F}_{F(2k-1)}(f(2k-1)|\{f(r)\}_{r=0}^{2t-2},\hat{D}_{all})}
\end{eqnarray}
Since $f_i(0)$ is randomly selected for all $i$, which is independent of dataset, there is $\mathscr{F}_{F(0)}(f(0)|D_{all}) = \mathscr{F}_{F(0)}(f(0)|\hat{D}_{all})$. First only consider $(2k-1)$-th iteration, since the primal variable is updated according to \eqref{eq:P_modify_2}, by KKT optimality condition:
\begin{eqnarray}\label{thmP2}
\epsilon_i(2k-1)=-\nabla O(f_i(2k-1),D_i) -2\lambda_i(2k-2)\nonumber \\  - \eta_i(2k-1) \sum_{j \in \mathscr{V}_i}(2f_i(2k-1)-f_i(2k-2)-f_j(2k-2))
\end{eqnarray}
Given $\{f(r)\}_{r=0}^{2k-2}$, $F_i(2k-1)$ and $E_i(2k-1)$ will be bijective $\forall i$, there is:
\begin{eqnarray}\label{thmP3}
&&\frac{\mathscr{F}_{F(2k-1)}(f(2k-1)|\{f(r)\}_{r=0}^{2k-2},D_{all})}{\mathscr{F}_{F(2k-1)}(f(2k-1)|\{f(r)\}_{r=0}^{2k-2},\hat{D}_{all})}
\nonumber \\&=& \prod^{N}_{v=1}\frac{\mathscr{F}_{F_v(2k-1)}(f_v(2k-1)|\{f_v(r)\}_{r=0}^{2k-2},D_v)}{\mathscr{F}_{F_v(2k-1)}(f_v(2k-1)|\{f_v(r)\}_{r=0}^{2k-2},\hat{D}_v)}
\nonumber\\&=& \frac{\mathscr{F}_{F_i(2k-1)}(f_i(2k-1)|\{f_i(r)\}_{r=0}^{2k-2},D_i)}{\mathscr{F}_{F_i(2k-1)}(f_i(2k-1)|\{f_i(r)\}_{r=0}^{2k-2},\hat{D}_i)}
\end{eqnarray}
Since two neighboring datasets $D_{all}$ and $\hat{D}_{all}$ only have at most one data point that is different, the second equality holds is because of the fact that this different data point could only be possessed by one node, say node $i$. Then there is $D_j = \hat{D}_j$ for $j \neq i$.

Given $\{f(r)\}_{r=0}^{2k-2}$, let $g_{k}(\cdot,D_i): \mathbb{R}^d \rightarrow \mathbb{R}^d $ denote the one-to-one mapping from $E_i(2k-1)$ to $F_i(2k-1)$ using dataset $D_i$. 
By Jacobian transformation, there is 
$\mathscr{F}_{F_i(2k-1)}(f_i(2k-1)|D_i) = \mathscr{F}_{E_i(2k-1)}(g^{-1}_{k}(f_i(2k-1),D_i))\cdot|\det(\textbf{J}(g^{-1}_{k}(f_i(2k-1),D_i)))|$
, where $g^{-1}_{k}(f_i(2k-1),D_i)$ is the mapping from $F_i(2k-1)$ to $E_i(2k-1)$ using data $D_i$ as shown in \eqref{thmP2} and $\textbf{J}(g^{-1}_{k}(f_i(2k-1),D_i))$ is the Jacobian matrix of it. Then 
\eqref{thmP1} yields:
\begin{eqnarray}\label{thmP5}
&&\frac{\mathscr{F}_{F(0:2K)}(\{f(r)\}_{r=0}^{2K}|D_{all})}{\mathscr{F}_{F(0:2K)}(\{f(r)\}_{r=0}^{2K}|\hat{D}_{all})}\nonumber \\&=& \prod^{K}_{k=1}\frac{\mathscr{F}_{E_i(2k-1)}(g^{-1}_{k}(f_i(2k-1),D_i))}{\mathscr{F}_{E_i(2k-1)}(g^{-1}_{k}(f_i(2k-1),\hat{D}_i))}
\nonumber\\ &&\cdot \prod^{K}_{k=1} \frac{|\det(\textbf{J}(g^{-1}_{k}(f_i(2k-1),D_i)))|}{|\det(\textbf{J}(g^{-1}_{k}(f_i(2k-1),\hat{D}_i)))|}
\end{eqnarray}
Consider the first part, $E_i(2k-1) \sim \exp\{-\alpha_i(k)||\epsilon||\}$, let $\hat{\epsilon}_i(2k-1) = g^{-1}_{k}(f_i(2k-1),\hat{D}_i)$ and ${\epsilon}_i(2k-1) = g^{-1}_{k}(f_i(2k-1),D_i)$
\begin{eqnarray}\label{thmP6}
&&\prod^{K}_{k=1}\frac{\mathscr{F}_{E_i(2k-1)}(g^{-1}_{k}(f_i(2k-1),D_i))}{\mathscr{F}_{E_i(2k-1)}(g^{-1}_{k}(f_i(2k-1),\hat{D}_i))}
\nonumber \\&= &\prod^{K}_{k=1} \exp(\alpha_i(k)(||\hat{\epsilon}_i(2k-1)|| - ||\epsilon_i(2k-1)||))
\nonumber \\&\leq& \exp(\sum^{K}_{k=1}\alpha_i(k)||\hat{\epsilon}_i(2k-1) - \epsilon_i(2k-1)||)
\end{eqnarray}
Without loss of generality, let $D_i$ and $\hat{D}_i$ be only different in the first data point, say $(x_i^1,y_i^1)$ and $(\hat{x}_i^1,\hat{y}_i^1)$ respectively. 
By \eqref{thmP2}, Assumptions 4 and the facts that $||x_i^n||_2 \leq 1$ (pre-normalization), $y_i^n \in \{+1,-1\}$.
\begin{eqnarray}\label{thmP7}
&&||\hat{\epsilon}_i(2k-1) - \epsilon_i(2k-1)|| \nonumber \\&=&||\nabla O(f_i(2k-1),\hat{D}_i)-\nabla O(f_i(2k-1),D_i)||
\nonumber \\&\leq& \frac{2C}{B_i}
\end{eqnarray}

\eqref{thmP6} can be bounded:
\begin{equation}\label{thmP8}
\prod^{K}_{k=1}\frac{\mathscr{F}_{E_i(2k-1)}(g^{-1}_{k}(f_i(2k-1),D_i))}{\mathscr{F}_{E_i(2k-1)}(g^{-1}_{k}(f_i(2k-1),\hat{D}_i))}
 \leq \exp(\sum^{K}_{k=1}\frac{2C\alpha_i(k)}{B_i})
\end{equation}

Consider the second part, the Jacobian matrix $\textbf{J}(g^{-1}_{k}(f_i(2k-1),D_i))$ is:
\begin{eqnarray}\label{thmP10}
&&\textbf{J}(g^{-1}_{k}(f_i(2k-1),D_i))\nonumber \\ &=& -\frac{C}{B_i}\sum_{n=1}^{B_i}\mathscr{L}''(y_i^n f_i(2k-1)^T x_i^n)x_i^n(x_i^n)^T
\nonumber \\&&-\frac{\rho}{N}\nabla^2 R(f_i(2k-1)) - 2\eta_i(2k-1) V_i\textbf{I}_d\nonumber 
\end{eqnarray}

Define
\begin{eqnarray}
G(k) &=& \frac{C}{B_i}(\mathscr{L}''(\hat{y}_i^1 f_i(2k-1)^T \hat{x}_i^1)\hat{x}_i^1(\hat{x}_i^1)^T \nonumber \\&&- \mathscr{L}''(y_i^1 f_i(2k-1)^T x_i^1)x_i^1(x_i^1)^T)\nonumber ~;\\
H(k) &= &-\textbf{J}(g^{-1}_{k}(f_i(2k-1),D_i))~.\nonumber
\end{eqnarray}

There is:
\begin{eqnarray}\label{thmP11}
 &&\frac{|\det(\textbf{J}(g^{-1}_{k}(f_i(2k-1),D_i)))|}{|\det(\textbf{J}(g^{-1}_{k}(f_i(2k-1),\hat{D}_i)))|}
\nonumber \\&=& \frac{|\det(H(k))|}{|\det(H(k)+G(k))|}
 = \frac{1}{|\det(I + H(k)^{-1}G(k))|}
 \nonumber \\ &=& \frac{1}{|\prod_{j=1}^r(1+\lambda_j(H(k)^{-1}G(k)))|}
\end{eqnarray}

where $\lambda_j(H(k)^{-1}G(k))$ denotes the $j$-th largest eigenvalue of $H(k)^{-1}G(k)$. Since $G(k)$ has rank at most 2, $H(k)^{-1}G(k)$ also has rank at most 2. By Assumptions 4 and 5, the eigenvalue of $H(k)$ and $G(k)$ satisfy
\begin{eqnarray}\label{thmP12}
&&\lambda_j(H(k)) \geq \frac{\rho}{N} + 2\eta_i(2k-1) V_i > 0 ~;\nonumber\\&&
-\frac{Cc_1}{B_i} \leq \lambda_j(G(k)) \leq \frac{Cc_1}{B_i}~.\nonumber
\end{eqnarray}

Implies
\begin{eqnarray}\label{thmP14}
-\frac{c_1}{\frac{B_i}{C}(\frac{\rho}{N}+2\eta_i(2k-1) V_i)}&\leq& \lambda_{j}(H(k)^{-1}G(k))
\nonumber\\&\leq& \frac{c_1}{\frac{B_i}{C}(\frac{\rho}{N}+2\eta_i(2k-1) V_i)}~.\nonumber 
\end{eqnarray}

Since $2c_1 < \frac{B_i}{C}(\frac{\rho}{N}+2\eta_i(1) V_i)$ and $\eta_i(2k-1)\leq \eta_i(2k+1)$ for all $k$, $2c_1 < \frac{B_i}{C}(\frac{\rho}{N}+2\eta_i(2k-1) V_i)$ holds. It implies the following,
\begin{equation*}
-\frac{1}{2}\leq \lambda_{j}(H(k)^{-1}G(k))
\leq \frac{1}{2}.
\end{equation*}

Since $\lambda_{\min}(H(k)^{-1}G(k)) > -1$, there is
\begin{eqnarray}
\frac{1}{|1+\lambda_{\max}(H(k)^{-1}G(k))|^2} & \leq& \frac{1}{|\text{det}(I+H(k)^{-1}G(k))|}  \nonumber \\&\leq &\frac{1}{|1+\lambda_{\min}(H(k)^{-1}G(k))|^2}~. \nonumber
\end{eqnarray}
Therefore, 
\begin{eqnarray}\label{thmP15}
 &&\prod^{K}_{k=1}\frac{|\det(\textbf{J}(g^{-1}_{k}(f_i(2k-1),D_i)))|}{|\det(\textbf{J}(g^{-1}_{k}(f_i(2k-1),\hat{D}_i)))|}
\nonumber \\ &\leq &\prod^{K}_{k=1}\frac{1}{|1-\frac{c_1}{\frac{B_i}{C}(\frac{\rho}{N}+2\eta_i(2k-1) V_i)}|^2}
\nonumber\\ &=& \exp(-\sum_{k=1}^{K}2\ln(1-\frac{c_1}{\frac{B_i}{C}(\frac{\rho}{N}+2\eta_i(2k-1) V_i)}))~.
\end{eqnarray}

Since for any real number $x \in [0,0.5]$, $-\ln(1-x)<1.4x$. \eqref{thmP15} can be bounded with a simper expression:
\begin{eqnarray}\label{thmP16}
 &&\prod^{K}_{k=1}\frac{|\det(\textbf{J}(g^{-1}_{k}(f_i(2k-1),D_i)))|}{|\det(\textbf{J}(g^{-1}_{k}(f_i(2k-1),\hat{D}_i)))|}
\nonumber \\& \leq& \exp(\sum_{k=1}^{K}\frac{2.8c_1}{\frac{B_i}{C}(\frac{\rho}{N}+2\eta_i(2k-1) V_i)})~. 
\end{eqnarray}

Combine \eqref{thmP8}\eqref{thmP16}, \eqref{thmP5} can be bounded:
\begin{eqnarray}
&&\frac{\mathscr{F}_{F(0:2K)}(\{f(r)\}_{r=0}^{2K}|D_{all})}{\mathscr{F}_{F(0:2K)}(\{f(r)\}_{r=0}^{2K}|\hat{D}_{all})} \nonumber\\ &\leq&\exp(\sum_{k=1}^{K}\frac{2C}{B_i}(\frac{1.4c_1}{(\frac{\rho}{N}+2\eta_i(2k-1) V_i)} + \alpha_i(k)))~.
\end{eqnarray}

Therefore, the total privacy loss during $T$ iterations can be bounded by any $\beta$:
\begin{equation*}
 \beta \geq \underset{i \in \mathscr{N}}{\max}\{\sum_{k=1}^{K}\frac{2C}{B_i}(\frac{1.4c_1}{(\frac{\rho}{N}+2\eta_i(2k-1) V_i)} + \alpha_i(k))\}~.
\end{equation*}

\section{Proof of Theorem \ref{thm:sample}}\label{App_4}
Let $\widetilde{O}(f) = C\mathcal{L}(f) + \frac{\rho}{2N}||f||^2 $ and
$\widetilde{f}_i = \text{argmin}_f \widetilde{O}(f) $. Let $f_i^{opt} = \text{argmin}_f {O}(f,D_i) $ be node $i$'s classifier trained with its own data. 
\begin{eqnarray*}
	\mathcal{L}(f_c^*) &=& \mathcal{L}(f_{ref})  + (\frac{\widetilde{O}(f_c^*)}{C}  -\frac{\widetilde{O}(\widetilde{f}_i )}{C} )\\&+&  (\frac{\rho}{2NC}||f_{ref}||^2- \frac{\rho}{2NC}||f_{c}^*||^2) \\&+&( \frac{\widetilde{O}(\widetilde{f}_i )}{C} - \frac{\widetilde{O}(f_{ref})}{C} )
\end{eqnarray*}
By \cite{NIPS2008_3400}, 
$\widetilde{O}(f_c^*) - \widetilde{O}(\widetilde{f}_i ) \leq (1+a)(O(f_c^*,D_i)-O(f_i^{opt},D_i)) + \mathcal{O}(\frac{C^2N\log(1/\delta)}{\rho B_i})$ holds $\forall a>0$ with probability $1-\delta$, where $\mathcal{O}$ is big-$\mathcal{O}$ notation.

Since $f_c^*$ is the centralized classifier trained with samples from all nodes, we assume the difference of empirical loss under two classifiers $f_c^*$ and $f_i^{opt}$ is bounded by $\nu>0$,   
i.e., $O(f_c^*,D_i)- O(f_i^{opt},D_i)\leq  \frac{\rho}{2N}(||f_c^*||^2-||f_i^{opt}||^2) + C\nu$. Moreover, $\widetilde{O}(\widetilde{f}_i )\leq \widetilde{O}(f_{ref}) $.
\begin{eqnarray*}
	\mathcal{L}(f_c^*) &\leq& \mathcal{L}(f_{ref})  + \mathcal{O}(\frac{CN\log(1/\delta)}{\rho B_i})\\
	&+& (1+a)(\frac{\rho}{2NC}||f_c^*||^2 -\frac{\rho}{2NC}||f_i^{opt}||^2 +\nu)
	\\&+& (\frac{\rho}{2NC}||f_{ref}||^2- \frac{\rho}{2NC}||f_{c}^*||^2) 
\end{eqnarray*}

We assume $\nu$ is relatively small as compared to other terms.
If choosing $a>0$ to be a sufficient small number such that $a||f_c^*||^2 -(1+a)||f_i^{opt}||^2 \leq 0$ and choosing
$\rho $ such that $\frac{\rho}{2NC}||f_{ref}||^2\leq \frac{\tau-\Delta_i(k)}{2}$, e.g., $\rho \leq \frac{NC(\tau-\Delta_i(k))}{||f_{ref}||^2}$, and if $B_i$ also satisfies $\mathcal{O}(\frac{CN\log(1/\delta)}{\rho B_i})\leq  \frac{\tau-\Delta_i(k)}{2}$, i.e., $$B_i\geq w \max_k\{ \frac{CN\log(1/\delta)}{\rho (\tau - \Delta_i(k))}\}\geq w \max_k\{ \frac{||f_{ref}||^2\log(1/\delta)}{(\tau - \Delta_i(k))^2}\}$$ for some constant $w$, then the following holds with probability $1-\delta$.
\begin{eqnarray*}
	\mathcal{L}(f_c^*) &\leq& \mathcal{L}(f_{ref})  +\tau - \Delta_i(k)
\end{eqnarray*}
Since $\mathcal{L}(f_i^{non}(2k-1))\leq \mathcal{L}(f_c^*) + \Delta_i(k)$, it implies that $\mathcal{L}(f_i^{non}(2k-1))\leq  \mathcal{L}(f_{ref})+\tau$ holds with probability $1-\delta$.

	\section{proof of Theorem \ref{thm:sample_priv}}\label{App_5}
Let $\widetilde{O}(f) = C\mathcal{L}(f) + \frac{\rho}{2N}||f||^2 $ and
$\widetilde{f}_i = \text{argmin}_f \widetilde{O}(f) $. Let $f_i^{opt} = \text{argmin}_f {O}(f,D_i) $ be node $i$'s classifier trained with its own data. Let $f_i^{privOpt} = \text{argmin}_f{O}^{priv}(f,D_i;\epsilon) = O(f,D_i) + \epsilon^Tf$ and $\widetilde{O}^{priv}(f;\epsilon) = \widetilde{O}(f) + \epsilon^Tf$. 

\begin{eqnarray*}
	\mathcal{L}(f_{new}^*) &=& \mathcal{L}(f_{ref})  + (\frac{\widetilde{O}(f_{new}^*)}{C}  -\frac{\widetilde{O}(f_{i}^{privOpt})}{C} )\\
	&+&(\frac{\widetilde{O}(f_{i}^{privOpt})}{C}  -\frac{\widetilde{O}(\widetilde{f}_i )}{C} )
	\\&+&  (\frac{\rho}{2NC}||f_{ref}||^2- \frac{\rho}{2NC}||f_{new}^*||^2) \\&+&( \frac{\widetilde{O}(\widetilde{f}_i )}{C} - \frac{\widetilde{O}(f_{ref})}{C} )
\end{eqnarray*}

For the new optimization problem, $f_{new}^*$ is centralized classifier trained with samples from all nodes while $f_{i}^{privOpt}$ is the classifier trained with samples from node $i$. We assume the difference of empirical loss under two classifiers $f_{new}^*$ and  $f_{i}^{privOpt}$ can be bounded by $\nu>0$, i.e., $\widetilde{O}(f_{new}^*)  -\widetilde{O}(f_{i}^{privOpt}) \leq \frac{\rho}{2N}(||f_{new}^*||^2-||f_i^{privOpt}||^2) + C\nu$.

By \cite{NIPS2008_3400}, $\widetilde{O}(f_{i}^{privOpt}) - \widetilde{O}(\widetilde{f}_{i}) \leq (1+a)( O(f_{i}^{privOpt},D_i) - O(f_{i}^{opt},D_i)) + \mathcal{O}( \frac{C^2N\log(1/\delta)}{\rho B_i})$ holds $\forall a>0$ with probability $1-\delta$.  By Lemma \ref{lemma:sample0}, $O(f_{i}^{privOpt},D_i)- O(f_{i}^{opt},D_i)\leq \frac{Nd^2}{\rho(\alpha_i(k))^2}(\log(d/\delta))^2$ holds with probability $1-\delta$. Therefore, $\widetilde{O}(f_{i}^{privOpt}) - \widetilde{O}(\widetilde{f}_{i}) \leq (1+a)( \frac{Nd^2}{\rho(\alpha_i(k))^2}(\log(d/\delta))^2) + \mathcal{O}( \frac{C^2N\log(1/\delta)}{\rho B_i})$ holds $\forall a>0$ with probability $1-2\delta$. 

Since $\widetilde{f}_i = \text{argmin}_f \widetilde{O}(f) $, implying $\widetilde{O}(\widetilde{f}_i ) \leq \widetilde{O}(f_{ref})$. The following holds $\forall a>0$ with probability $1-2\delta$, 

\begin{eqnarray*}
	\mathcal{L}(f_{new}^*) &\leq& \mathcal{L}(f_{ref})  + \nu+ \mathcal{O}( \frac{CN\log(1/\delta)}{\rho B_i})\\
	&+&(1+a) \frac{Nd^2}{C\rho(\alpha_i(k))^2}(\log(d/\delta))^2
	\\&+&  (\frac{\rho}{2NC}||f_{ref}||^2- \frac{\rho}{2NC}||f_i^{privOpt}||^2) 
\end{eqnarray*}

We assume $\nu$ is relatively small as compared to other terms.
If choosing $\rho$ such that $\frac{\rho}{2NC}||f_{ref}||^2 \leq \frac{1}{2}(\tau - \Delta_i^{new}(k))$, i.e., $\rho \leq \frac{NC(\tau - \Delta_i^{new}(k))}{||f_{ref}||^2}$, and if $B_i$ also satisfies $((1+a) \frac{Nd^2}{C(\alpha_i(k))^2}(\log(d/\delta))^2+ \mathcal{O}( \frac{CN\log(1/\delta)}{ B_i}))\leq \frac{\rho(\tau-\Delta_i^{new}(k))}{2}$, i.e., $B_i\geq w\frac{CN\log(1/\delta)}{\frac{\rho(\tau-\Delta_i^{new}(k))}{2}-(1+a) \frac{Nd^2}{C(\alpha_i(k))^2}(\log(d/\delta))^2}$ for some $a>0$ and constant $w$. Then  $\mathcal{L}(f_{new}^*) \leq \mathcal{L}(f_{ref}) + \tau -\Delta_i^{new}(k)$ holds with probability $1-2\delta$. Plug in $\rho =\frac{NC(\tau - \Delta_i^{new}(k))}{||f_{ref}||^2}$ and re-organize gives:
\begin{eqnarray*}
	B_i\geq w\max_k\{\frac{CN\log(1/\delta)}{\frac{NC(\tau-\Delta_i^{new}(k))^2}{2||f_{ref}||^2}-(1+a) \frac{Nd^2}{C(\alpha_i(k))^2}(\log(d/\delta))^2}\}
\end{eqnarray*}
Since $\mathcal{L}(f_i^{new}(2k-1))\leq \mathcal{L}(f_{new}^*) + \Delta_i^{new}(k)$, it implies that $\mathcal{L}(f_i^{new}(2k-1)) \leq \mathcal{L}(f_{ref}) + \tau$ holds with probability $1-2\delta$.

\begin{lemma}\label{lemma:sample0}
	Let $f_{i}^{privOpt} = \text{argmin}_f O(f,D_i)+\epsilon^Tf$ and $f_{i}^{opt} = \text{argmin}_f O(f,D_i)$ be outputs at iteration $2k-1$, then $O(f_{i}^{privOpt},D_i)- O(f_{i}^{opt},D_i)\leq  \frac{Nd^2}{\rho(\alpha_i(k))^2}(\log(d/\delta))^2$ holds with probability $1-\delta$. 
\end{lemma}
\begin{proof}
	There is $O(f_{i}^{privOpt},D_i) \leq O(f_{i}^{opt},D_i) + \epsilon^T(f_i^{opt}-f_i^{privOpt})$. By Lemma \ref{lemma:sample1}, since $O(f,D_i)$ and $O(f,D_i)+\epsilon^Tf$ are $\frac{\rho}{N}$-strongly convex, $||f_i^{opt}-f_i^{privOpt}||\leq \frac{N}{\rho}||\epsilon||$ holds. By Lemma \ref{lemma:sample2}, with probability $1-\delta$, $||\epsilon||\leq \frac{d}{\alpha_i(k)}\log(d/\delta)$. Therefore, $O(f_{i}^{privOpt},D_i)- O(f_{i}^{opt},D_i)\leq ||\epsilon|| ||f_i^{opt}-f_i^{privOpt} ||\leq  \frac{Nd^2}{\rho(\alpha_i(k))^2}(\log(d/\delta))^2$ holds with probability $1-\delta$. 
\end{proof}

\begin{lemma}\label{lemma:sample1}
	\cite{chaudhuri2011} Let $G(f)$, $g(f)$ be two vector-valued functions, which are continuous and differentiable at all points. Moreover, let $G(f)$ and $G(f)+g(f)$ be $\lambda$-strongly convex. If $f_1 = \text{argmin}_f G(f)$ and $f_2 = \text{argmin}_f G(f)+g(f)$, then $||f_1-f_2||\leq \frac{1}{\lambda}\max_f ||\nabla g(f)||$.
\end{lemma}

\begin{lemma}\label{lemma:sample2}
	\cite{chaudhuri2011} Let X be a random variable drawn from distribution $\Gamma(k,\theta)$, where $k$ is an integer, then $Pr(X<k\theta\log(k/\delta))\geq 1-\delta$.
\end{lemma}



\ifCLASSOPTIONcaptionsoff
  \newpage
\fi



%

\end{document}